\documentclass[accepted]{uai2026}
\usepackage[american]{babel}
\usepackage[utf8]{inputenc}
\usepackage[T1]{fontenc}
\usepackage{natbib}
\usepackage{hyperref}
\hypersetup{colorlinks=true,linkcolor=blue!50!black,citecolor=blue!50!black,urlcolor=blue!50!black}
\usepackage{url}
\usepackage{wrapfig} 
\usepackage{algorithm}
\usepackage{algorithmic}
\usepackage{booktabs}
\usepackage{amsfonts}     
\usepackage{amsmath}
\usepackage{amsthm}
\usepackage{color,xcolor}
\usepackage{mathtools}
\usepackage{nicefrac}
\usepackage{tikz}
\usetikzlibrary{positioning,arrows.meta,calc, matrix}
\usepackage{siunitx}
\usepackage{tabularray}
\usepackage{caption}
\usepackage{adjustbox}
\usepackage{enumitem}
\theoremstyle{plain}
\newtheorem{theorem}{Theorem}[section]

\newtheorem{lemma}[theorem]{Lemma}

\theoremstyle{definition}
\newtheorem{definition}[theorem]{Definition}
\newtheorem{assumption}[theorem]{Assumption}
\theoremstyle{remark}
\newtheorem{remark}[theorem]{Remark}
\newtheorem{example}{Example}
\newcommand{\A}{\mathcal{A}}

\newcommand{\X}{\mathcal{X}}

\newcommand{\R}{\mathbb{R}}
\newcommand{\E}{\mathbb{E}}
\newcommand{\Ss}{\mathcal{S}}

\newcommand{\Lvert}{\left\lvert}
\newcommand{\Rvert}{\right\rvert}
\newcommand{\LVert}{\left\lVert}
\newcommand{\RVert}{\right\rVert}
\newcommand{\Law}{\mathrm{Law}}

\newcommand{\norm}[1]{\left\lVert #1\right\rVert}
\newcommand{\abs}[1]{\left\lvert #1\right\rvert}

\DeclarePairedDelimiterX{\gendivx}[2]{(}{)}{%
  #1\;\delimsize\|\;#2%
}

\usepackage{microtype}
\usepackage{graphicx}
\usepackage{subfigure}
\usepackage{amssymb}
\usepackage[capitalize,noabbrev]{cleveref}
\usepackage[textsize=tiny]{todonotes}
\usepackage{accents}

\title{Joint MDPs and Reinforcement Learning in Coupled-Dynamics Environments}

\author[1]{Ege~C.~Kaya}
\author[1]{Mahsa~Ghasemi}
\author[1]{Abolfazl~Hashemi}

\affil[1]{
    Elmore Family School of Electrical and Computer Engineering\\
    Purdue University\\
    West Lafayette, IN, USA
}
  
\begin{document}
\maketitle

\begin{abstract}
Many distributional quantities in reinforcement learning are intrinsically joint across actions, including distributions of gaps and probabilities of superiority. However, the classical Markov decision process (MDP) formalism specifies only marginal laws and leaves the joint law of counterfactual one-step outcomes across multiple possible actions at a state unspecified. We study \textit{coupled-dynamics environments} with a multi-action generative interface which can sample counterfactual one-step outcomes for multiple actions under shared exogenous randomness. We propose \textit{joint MDPs} (JMDPs) as a formalism for such environments by augmenting an MDP with a multi-action sample transition model which specifies a coupling of one-step counterfactual outcomes, while preserving standard MDP interaction as marginal observations. 
We adopt a local one-step coupling regime in which each queried state admits a shared-randomness outcome table and future tables are freshly sampled recursively. In this regime, we derive Bellman operators for $n$th-order return moments, providing dynamic programming and incremental algorithms with convergence guarantees.
\end{abstract}
\section{Introduction}\label{sec:intro}
Distributional reinforcement learning (DRL) studies the discounted return random variable (RV) $Z^\pi(s, a)$ belonging to a policy $\pi$, and seeks to represent or estimate its distributional properties beyond its expectation. Formally,
\begin{equation}
\begin{gathered}
Z^\pi(s, a) = \sum_{t=0}^\infty \gamma^t R_t,\; S_0 =s, A_0 = a, \; A_{t} \sim \pi(\cdot \mid S_t) \\ R_t \sim P_R(\cdot \mid S_t, A_t),\; S_{t+1}\sim P_S(\cdot \mid S_t, A_t).
\end{gathered}
\end{equation}
DRL methods customarily aim to learn per-action marginal return laws at state $s$, i.e., $\{\Law(Z^\pi(s, a))\}_{a \in \A}$ \citep{c51, DRL-textbook}. However, a range of distributional quantities that are of interest in decision-making \citep{artzner1999coherent, rockafellar2000optimization} are not functions of per-action marginals alone. For actions $a, \Tilde{a} \in \A$, examples include:
\begin{enumerate}[leftmargin=*, itemsep=2pt]
\item the \textit{gap} RV $G^\pi(s; a, \Tilde{a}) := Z^\pi(s, a) - Z^\pi(s, \Tilde{a})$ and its distribution $\Law(G^\pi(s; a, \Tilde{a}))$,
\item any tail functional of $G^\pi(s; a, \Tilde{a})$, such as the quantile $q_\alpha(G^\pi(s; a,\Tilde{a}))$ or the $\mathrm{CVaR}_\alpha(G^\pi(s; a,\Tilde{a}))$,
\item the \textit{probability of superiority} $\mathbb{P}(Z^\pi(s, a) > Z^\pi(s, \Tilde{a}))$.
\end{enumerate}
These quantities depend on the joint law of $(Z^\pi(s, a)), \allowbreak Z^\pi(s, \Tilde{a}))$ and thus require knowledge of coupling structure.

\textbf{A limitation of the Markov decision process (MDP) formalism.} MDPs specify the marginal distributions of the reward and successor state under each action. However, they make no specification about the joint distribution of \textit{counterfactual, one-step outcomes across multiple actions at a state}. Therefore, joint objects such as $\Law(Z^\pi(s, a) - Z^\pi(s, \Tilde{a}))$ are not well-defined without adopting additional coupling conventions \citep{wiltzer2024action}.

\textbf{Coupled-dynamics environments.} The preceding limitation is substantive in environments where counterfactual one-step outcomes under multiple actions are operationally meaningful and sampleable. A canonical setting is scenario-based simulation, where an exogenous disturbance $U$ is realized at each decision point and one-step outcomes are functions of $(s, a, U)$. Evaluating multiple candidate actions under the same realized disturbance is standard in simulation optimization \citep{glasserman2004monte, shapiro2021lectures} and yields a multi-action generative interface. Examples include logistics, dispatch, inventory, energy, robotics, and network simulators where actions are compared under the same demand, traffic, load, cloned-state or latency shock.

\textbf{Contribution and scope.} 
\begin{itemize}[leftmargin=*]
\item We propose \textit{joint MDPs} (JMDPs) as a formalism for coupled-dynamics environments. A JMDP augments an MDP with a multi-action sample transition model that specifies a coupling of one-step counterfactual outcomes. 
\item We focus exclusively on \textit{policy evaluation}: For a fixed policy $\pi$, we derive moment Bellman operators up to arbitrary finite order and provide both dynamic programming (DP) and incremental estimation algorithms with convergence guarantees, along with computable Bellman residual certificates \citep{scherrer2010should}.
\item \textbf{One-step coupling regime.} We adopt a local one-step coupling regime. Branches that query the same state share the current outcome table, while distinct successor states use independent fresh tables. The Bellman recursion then propagates mixed moments over successor pairs without constructing fully coupled counterfactual trajectory trees.
\end{itemize}

\section{Related Works}\label{sec:related}
\textbf{MDP formalism and generative interfaces.} In the standard MDP formalism, the environment is specified by a per-action reward law and a per-action transition kernel \citep{puterman, BartoSutton, DRL-textbook}. Interaction reveals only the reward and next state of the executed action. This is sufficient for objectives depending on \textit{marginal} state-action return laws (e.g., expected value or a single-action return distribution), but it does not specify a joint law over counterfactual one-step outcomes across actions. We formalize a setting in which the environment exposes a \textit{multi-action generative interface} returning counterfactual one-step outcomes for multiple queried actions under shared exogenous randomness. Such shared-randomness interfaces are standard in common-random-number Monte Carlo (MC) comparison and scenario-based decision models
\citep{glasserman2004monte, shapiro2021lectures}.

\textbf{Distributional and multivariate RL}. DRL studies the law of the scalar return and Bellman operators on spaces of probability measures, beginning with the distributional perspective of \citet{c51} and subsequent practical approximations \citep{c51, qr-dqn, iqn, FQF, nguyen2021distributional}. These methods learn per-state action \textit{marginal} return distributions and do not define joint RVs $(Z^\pi(s, a))_{a \in \A}$. A complementary line studies multivariate return distributions induced by vector-valued rewards and cumulants, including Bellman-GAN \citep{bellman-gan}, MD3QN \citep{zhang2021distributional}, and the recent foundations and convergent DP/TD methods of \citet{wiltzer2024foundations}. Our setting is orthogonal: even with scalar rewards, we study joint structure across counterfactual actions induced by coupled-environment randomness, which is not captured by standard multivariate DRL formulations.

\textbf{Moment-based evaluation, risk sensitivity, and gap/advantage distributions.} Return-moment policy evaluation has a long history \citep{sobel}, including DP characterizations of variance and TD estimators for variance and related quantities \citep{mv-mdp, tamar2, estimating-variance}. Coherent risk measures such as CVaR have likewise been incorporated via risk-aware Bellman operators \citep{cvar-mdp}. These works remain within a single action's return law. By contrast, our target quantities are intrinsically joint across actions and require mixed moments encoding cross-action dependence. This distinction is especially visible for gap and advantage distributions: if only marginals are given, the law of $Z^\pi(s, a) - Z^\pi(s, \Tilde{a})$ depends on an unspecified coupling, motivating canonical-coupling axioms \citep{wiltzer2024action}. Our approach instead models environments in which the coupling is part of the environment itself. Under our one-step coupling regime, gap and advantage RVs are well-defined by construction, and their moments are computable from learned mixed moments.

\textbf{Counterfactual modeling perspectives.} The multi-action generative interface admits a counterfactual interpretation. At a state, the environment reveals multiple potential one-step outcomes under shared exogenous randomness. This connects to potential-outcome and structural causal model (SCM) perspectives in RL \citep{lu2020sample, amitai2024explaining}. We do not address identifiability or learning such an SCM from data, rather, assuming simulator access to counterfactual one-step queries, we ask what formalism and policy-evaluation theory are appropriate. A POMDP or SCM is natural when latent variables must be inferred from ordinary trajectories. Here, the shared-randomness counterfactual table is queryable, and even if a latent $U$ renders outcomes conditionally deterministic, the unconditional gap under a shared $U$ remains a joint object rather than a marginal one.

Overall, prior DRL and risk-sensitive RL methods either model marginal per-action return laws or multivariate returns arising from vector-valued rewards along a realized trajectory. Our work isolates a different axis of jointness: across actions at a fixed state, induced by coupled-environment randomness. We propose joint MDPs as a formalism and algorithmic basis for policy evaluation of both marginal and joint return moments, enabling principled treatment of intrinsically joint distributional quantities.

\section{Preliminaries}\label{sec:preliminaries}
\begin{definition}[MDP] 
An MDP is a quintuple $(\Ss, \allowbreak\A, \allowbreak P_R, \allowbreak P_S, \allowbreak \gamma)$, where $\Ss$ is a finite state space, $\A$ is a finite action space with $N = \abs{\A}$, $\gamma \in (0, 1)$ is the discount factor, $P_R(\cdot \mid s, a)$ is a reward kernel over $[0, 1]$ and $P_S(\cdot\mid s,a)$ is a transition kernel over $\Ss$.
\end{definition}

A policy is a Markov kernel $\pi(\cdot \mid s) \in \Delta(\A)$. For measurable spaces $(\Omega_i, \mathcal{F}_i)$, let $\mathcal{P}(\Omega_i)$ denote the set of probability measures over $\Omega_i$. An $m$-variate distribution $\nu \in \mathcal{P}(\Omega_1 \times \ldots \Omega_m)$ is a coupling of its marginals $\{ \nu_i\}_{i=1}^m$, where $\nu_i$ is the $i$th coordinate marginal. We write $[m]$ for $\{1, \ldots, m\}$. To conclude the section, let us recall a standard convention, used throughout RL, for describing the interaction with an environment via a generative model.

\begin{definition}[Sample transition model (STM) \citep{DRL-textbook}]\label{def:stm}
Let $\xi \in \Delta(\Ss)$ and $\pi$ be a policy. The STM is the joint distribution of $(S, A, R, S')$ generated by
\begin{equation}
\begin{gathered}
S \sim \xi, \quad A \mid S \sim \pi(\cdot \mid S), \\
R \mid S, A \sim P_R(\cdot \mid S, A), \quad
S' \mid S, A \sim P_S(\cdot \mid S, A).
\end{gathered}
\end{equation}
\end{definition}

\section{Coupled-dynamics environments and joint MDPs}
Many distributional quantities in RL are intrinsically joint across actions at a fixed state, including gaps and advantages, probabilities of superiority, and tail functionals of such quantities. Formalizing such quantities requires a joint law for \textit{counterfactual} one-step outcomes that would occur under different actions taken from the same state, under the same realization of exogenous randomness. The MDP formalism does not account for this, since it only specifies marginal one-step laws for each action and leaves the joint counterfactual structure unspecified. Consequently, environments that agree on all one-step marginals when modeled using an MDP may still disagree on joint-action quantities.

We now present a multi-action variant of the STM, the $m$-joint STM ($m$-JSTM). Informally, an $m$-JSTM allows the learner to query $m$ distinct actions at the same state, and obtain $m$ coupled counterfactual one-step outcomes.
\begin{definition}[$m$-JSTM]\label{def:jstm} For an integer $m\ge 1$, an $m$-JSTM is a kernel 
\begin{equation}
\begin{gathered}
\mathcal{J}_m(\cdot \mid s, a_{1:m}) \in \mathcal{P}\big(([0, 1] \times \Ss)^m\big),\\ s\in\Ss, a_{1:m}
\in \A^m \text{ distinct},
\end{gathered}
\end{equation}
such that each coordinate marginal matches the classical one-step marginals. Concretely, if $((R_i, S_i'))_{i=1}^m \sim \mathcal{J
}_m(\cdot \mid s, a_{1:m})$, then for each coordinate $i \in [m]$,
\begin{equation}
\begin{gathered}
(R_i, S'_i) \sim \big(P_R(\cdot \mid s, a_i), P_S(\cdot \mid s, a_i) \big).
\end{gathered}
\end{equation}
\end{definition}

Specifically, when $m =N$ and $a_{1:N}$ is a fixed ordering of $\A$, $\mathcal{J}_m(\cdot \mid s, a_{1:N})$ specifies a joint coupling across all action at $s$. The additional expressiveness provided by an $m$-JSTM is the \textit{dependence} it specifies across counterfactual outcomes for different actions. When this dependence is nontrivial, the MDP formalism omits information that is directly relevant for joint distributional questions. This motivates treating the coupling as part of the model.
\begin{definition}[Coupled-dynamics environment]
An environment with marginal kernels $(P_R, P_S)$ is a \textit{coupled-dynamics} environment if it admits an $m$-JSTM $\mathcal{J}_m$ for some $m \ge 2$ such that for some query $(s, a_{1:m})$, $\mathcal{J}_m(\cdot \mid s, a_{1:m})$ is \textbf{not} the product coupling of its marginals. Equivalently, the environment contains joint one-step counterfactual structure that is not determined by $(P_R, P_S)$ alone, and is therefore lost under a purely marginal MDP description.
\end{definition}

The following examples illustrate the issue: Identical MDP marginals can conceal different joint structures.
\begin{example}
Let $\A = \{1, 2\}$ and fix $s \in \Ss$. Let $U \sim \mathrm{Bernoulli}(1/2)$, and let the counterfactual rewards be $R^{(1)} = U$, $R^{(2)} = 1- U$. Each action's marginal reward law is $\mathrm{Bernoulli}(1/2)$, while their joint law is perfectly anti-correlated. $\mathbb{P}(R^{(1)} > R^{(2)}) = 1/2$ and $\E[R^{(1)}R^{(2)}]=0$, neither fact determined by the marginals alone.
\end{example}

\begin{example}
Let $\Ss = \{0, 1\}$, $\A = \{1, 2\}$, and fix $s \in \Ss$. Let $U \sim \mathrm{Bernoulli}(1/2)$, and the counterfactual next states be
\begin{enumerate}[leftmargin=*, itemsep=2pt,align=left]
\item[(2.1)] $S'^{(1)}= S'^{(2)} =U$. Each action's marginal transition law is uniform on $\Ss$, while $S'^{(1)}= S'^{(2)}$.
\item[(2.2)] $S'^{(1)} = U$ and $S'^{(2)}= 1-U$. Each action's marginal transition law is uniform on $\Ss$, while $S'^{(1)}=1- S'^{(2)}$ .
\end{enumerate}
\end{example}

We now formalize a joint extension of the MDP model, appropriate for modeling coupled-dynamics environments.
\begin{definition}[JMDP]
A \textit{JMDP} is a quadruple $(\Ss, \allowbreak\A, \allowbreak\gamma, \allowbreak\mathcal{J})$, where $\Ss$ is a finite state space, $\A$ is a finite action space with $N = \abs{\A}$, $\gamma \in (0, 1)$ is the discount factor, and $\mathcal{J}(\cdot \mid s) \in \mathcal{P}(([0, 1] \times \Ss)^N)$ is a Markov kernel over counterfactual one-step outcome tables $( (R^{(a)}, S'^{(a)}))_{a \in \A}$. 

At time $t$, conditional on $S_t = s$, the environment samples a one-step outcome table $((R^{(a)}, S'^{(a)}))_{a \in \A} \sim \mathcal{J}(\cdot \mid s)$. The agent selects action $A_t$ to execute, the realized transition is given by the executed action: $(R_t, S_{t+1}) = (R_t^{(A_t)}, S_{t+1}^{(A_t)}).$ The remaining coordinates of the sampled outcome table are counterfactual and do not affect the realized trajectory.
\end{definition}
\begin{remark}
Every JMDP induces a marginal MDP description of the environment when we take coordinate marginals of the kernel $\mathcal{J}$. To see this, consider the collection of $1$-JSTMs $\mathcal{J}_1(\cdot \mid s, a)$ for each $(s, a)$ such that
\begin{equation}
\begin{gathered}
\mathcal{J}_{1}(\cdot\mid s, a):=\Law\big( (R^{(a)}, S'^{(a)})\big), \\
\text{where } \big((R^{(b)}, S'^{(b)}) \big)_{b \in \A} \sim \mathcal{J}(\cdot\mid s).
\end{gathered}
\end{equation}
Then, the reward and transition marginals $P_R(\cdot \mid s, a)$ and $P_S(\cdot \mid s, a)$ of $\mathcal{J}_{1}$ will constitute the reward and transition kernels of a classical MDP description of the environment.

In a similar vein, fixing $m \ge 2$, for any state $s$ and any $m$-tuple of distinct actions $a_{1:m}$, we can induce an $m$-JSTM $\mathcal{J}_m$ as the joint law of the $m$ queried coordinates:
\end{remark}
\begin{equation}
\begin{gathered}
\mathcal{J}_{m}(\cdot\mid s, a_{1:m}):=\Law\big( (R^{(a)}, S'^{(a)})_{a \in a_{1:m}}\big), \\
\text{where } \big((R^{(b)}, S'^{(b)}) \big)_{b \in \A} \sim \mathcal{J}(\cdot\mid s).
\end{gathered}
\end{equation}

To summarize, a coupled-dynamics environment is one where the marginal description of an MDP formalism does not determine the joint law of counterfactual one-step outcomes across actions. JMDPs model the missing structure via a kernel $\mathcal{J}(\cdot \mid s)$ over counterfactual outcome tables, from which the standard MDP interaction model STM and the multi-action $m$-JSTM interface are obtained by marginalization. The JMDP formalism retains the dependence information discarded by an MDP description of the environment, essential for joint distributional questions.

The JMDP definition specifies how a counterfactual one-step outcome table is sampled at each visited state, and how a single executed action generates the realized transition. Before developing policy evaluation algorithms for joint quantities, it remains to formalize what we have referred to as the \textit{one-step coupling} regime. The coupling is local in time, i.e., the current table shares exogenous randomness across actions at the queried state, and the next step samples fresh tables. If two coordinates later query the same successor state, they are coupled again through that local table. If they are at distinct states, their one-step randomness is independent. This matches multi-action simulators while also avoiding fully-coupled, exponentially-expanding counterfactual trajectory trees. 

\begin{assumption}[Exogenous noise representation \citep{bertsekas1996stochastic, ng2013pegasus}]\label{ass:exo}
We assume that there exist an i.i.d. exogenous sequence $(U_t)_{t\ge0}$ over a measurable space $\mathcal{U}$, and measurable maps
\begin{equation}
\begin{gathered}
g: \Ss \times \A \times \mathcal{U} \to [0, 1], \quad
h :\Ss \times \A \times \mathcal{U} \to \Ss,
\end{gathered}
\end{equation}
such that at any time $t$ and state $s$, the counterfactual one-step outcomes for any action $a \in \A$ are related by 
\begin{equation}
\begin{gathered}
R^{(a)}_t = g(s, a, U_t), \quad S^{(a)}_{t+1} = h(s, a, U_t).
\end{gathered}
\end{equation}
Thus, the full outcome table sampled by the JMDP kernel can alternatively be expressed as
\begin{equation}
\big((R_t^{(a)}, S^{(a)}_{t+1}) \big)_{a\in \A} = \big((g(S_t, a, U_t), h(S_t, a, U_t))\big)_{a\in \A}.
\end{equation}
\end{assumption}
Assumption~\ref{ass:exo} means that after the transition to successor states after time $t$, subsequent one-step outcome tables are generated using fresh exogenous variables, independent of $U_t$. Dependence is therefore mediated by local table draws and not by a single latent variable which is shared for an entire counterfactual tree.

We now define the main object of the rest of our inquiry.
\begin{definition}[Joint return vector]
Let $\pi$ be a stationary policy. For any state $s$, we define the \textit{joint return vector} belonging to $\pi$ as $Z^\pi(s) := ( Z^\pi(s, a))_{a \in \A}$.
\end{definition}
Under the one-step coupling regime, the components of $Z^\pi(s)$ are coupled through the shared outcome table at the initial state $s$, and then recursively through the pair process over successor states.

\section{Joint iterative policy evaluation}

We now develop DP and incremental policy evaluation schemes for JMDPs. Throughout this section, we fix a policy $\pi$, and the goal is to estimate joint moments of $Z^\pi(s)$, including mixed moments.  We start with the case of moments up to $2$nd order, displaying an approachable and interpretable instance of the general case of $n$th-order moments. Proofs are deferred to Appendix \ref{app:proofs}.

Let $\X:=\Ss \times \A$ denote the state-action space. For each $(s, a) \in \X$, we define the first moments
\begin{equation}
\mu^\pi(s, a) := \E\left[Z^\pi(s, a)\right],
\end{equation}

and for each $(s, a, \Tilde{s}, \Tilde{a}) \in \X^2$, the second moments
\begin{equation}
\Sigma^\pi(s, a, \Tilde{s}, \Tilde{a}) := \E\left[Z^\pi(s, a)\,Z^\pi(\Tilde{s}, \Tilde{a})\right].
\end{equation}

The algorithms defined below will operate on candidate moment collections $M = (M_\mu, M_\Sigma)$, where
\begin{equation}
\begin{gathered}
M_\mu\in \{f:\X \to \R\}, \quad
M_\Sigma \in \{f(s, a, \Tilde{s},\Tilde{a}): \X^2 \to \R\}.
\end{gathered}
\end{equation}
Thus, $M_\mu(s, a)$ will be a candidate value for $\mu^\pi(s,a)$, and $M_\Sigma(s, a, \Tilde{s}, \Tilde{a})$ for $\Sigma^\pi(s, a, \Tilde{s}, \Tilde{a})$. The introduced algorithms will map any such dimensionally consistent moment collection $M$ to the true collection of moments $M^\pi_2:=(\mu^\pi, \Sigma^\pi)$.

We now define the $2$nd-order joint Bellman operator.

\begin{definition}[$2$nd-order joint Bellman operator] The $2$nd-order joint Bellman operator $T^\pi_2$ maps a moment collection $M$ to $T^\pi_2M$ coordinatewise through the equations
\begin{equation}\label{eq:Tmean}
\begin{gathered}
(T^\pi_2M)_\mu(s, a) := \E\big[ R^{(a)} + \gamma M_\mu\big(S'^{(a)}, A'^{(a)}\big) \mid S= s\big]
\end{gathered}
\end{equation}
for each $(s, a)$. Furthermore, for each $(s, a, \Tilde{s},\Tilde{a})$,
\begin{equation}\label{eq:Tcov}
\begin{gathered}
(T^\pi_2 M)_\Sigma(s, a, \Tilde{s}, \Tilde{a}) := \E \big[R^{(a)}\Tilde{R}^{(\Tilde{a})}\\+\gamma R^{(a)} M_\mu\big(\Tilde{S}'^{(\Tilde{a})},\Tilde{A}'^{(\Tilde{a})}\big)+  \gamma \Tilde{R}^{(\Tilde{a})} M_\mu\big(S'^{(a)},A'^{(a)}\big) \\+ \gamma^2 M_\Sigma\big(S'^{(a)}, A'^{(a)}, \Tilde{S}'^{(\Tilde{a})}, \Tilde{A}'^{(\Tilde{a})}\big)  \mid (S, \Tilde{S}) = (s, \Tilde{s}) \big].
\end{gathered}
\end{equation}
\end{definition}
It is instructional to consider the joint laws with respect to which we take the expectation in these equations. In \eqref{eq:Tmean}, this is the joint law of $(R^{(a)}, S'^{(a)}, A'^{(a)})$, where
\begin{equation}
(R^{(a)}, S'^{(a)}) \sim J_1(\cdot \mid s, a),\quad A'^{(a)} \sim \pi(\cdot\mid S'^{(a)}),
\end{equation}
where $\mathcal{J}_1$ is a $1$-JSTM induced from the JMDP kernel $\mathcal{J}$.

This is also the case in the ``diagonal'' scenario of \eqref{eq:Tcov} ($s = \Tilde{s}$ and $a = \Tilde{a}$), where the expression simplifies to
\begin{equation}
\begin{gathered}
(T^\pi_2 M)_\Sigma(s, a, s, a) \\ = \E\big[\big( R^{(a)}\big)^2+2\gamma R^{(a)}M_\mu\big(S'^{(a)}, A'^{(a)}\big) \\+ \gamma^2 M_\Sigma\big(S'^{(a)}, A'^{(a)}, S'^{(a)}, A'^{(a)}\big) \mid S = s \big],
\end{gathered}
\end{equation} 

When $a\ne \Tilde{a}$, the joint law is of $(R^{(a)},\allowbreak S'^{(a)},\allowbreak A'^{(a)},\allowbreak \Tilde{R}^{(\Tilde{a})},\allowbreak \Tilde{S}'^{(\Tilde{a})}, \Tilde{A}'^{(\Tilde{a})})$. When we have $s \ne \Tilde{s}$,
\begin{equation}
\begin{gathered}
(R^{(a)}, S'^{(a)}) \sim \mathcal{J}_1(\cdot \mid s, a), \quad (\Tilde{R}^{(\Tilde{a})}, \Tilde{S}'^{(\Tilde{a})}) \sim \mathcal{J}_1(\cdot \mid \Tilde{s}, \Tilde{a}),\\
A'^{(a)} \sim \pi (\cdot \mid S'^{(a)}), \quad \Tilde{A}'^{(\Tilde{a})} \sim \pi(\cdot \mid \Tilde{S}'^{(\Tilde{a})}).
\end{gathered}
\end{equation}
$\mathcal{J}_1$ is a $1$-JSTM induced from $\mathcal{J}$, the draws are independent.

Crucially, in the ``same-state'' case where $s = \Tilde{s}$, we have
\begin{equation}
\begin{gathered}
\big(R^{(a)}, S'^{(a)}, \Tilde{R}^{(\Tilde{a})}, \Tilde{S}'^{(\Tilde{a})}\big) \sim \mathcal{J}_2\big(\cdot \mid s, (a, \Tilde{a})\big),\\
A'^{(a)} \sim \pi (\cdot \mid S'^{(a)}), \quad \Tilde{A}'^{(\Tilde{a})} \sim \pi(\cdot \mid \Tilde{S}'^{(\Tilde{a})}),
\end{gathered}
\end{equation}
where $\mathcal{J}_2$ is a $2$-JSTM induced from $\mathcal{J}$.

Note that, by construction, $M^{\pi}_2$ is a fixed point of $T^\pi_2$, since
\begin{equation}
\begin{gathered}
\mu^\pi(s, a) = \E \big[R^{(a)}+\gamma Z^\pi\big(S'^{(a)},A'^{(a)}\big) \mid S = s \big],\\
\Sigma^\pi(s, a, \Tilde{s},\Tilde{a}) = \E \big[\big(R^{(a)}+\gamma Z^\pi\big(S'^{(a)},A'^{(a)}\big)\big) \\
\cdot\big(\Tilde{R}^{(\Tilde{a})}+\gamma Z^\pi\big(\Tilde{S}'^{(\Tilde{a})},\Tilde{A}'^{(\Tilde{a})}\big)\big)\mid (S, \Tilde{S}) = (s, \Tilde{s}) \big].
\end{gathered}
\end{equation}

We now prove that $T^\pi_2$ is a contraction mapping.

\begin{lemma}\label{lem:2contraction}
Let $\lambda := 2/(1-\gamma)$. For any moment collection $M$, define the norm
\begin{equation}
\norm{M}_{\lambda} := \max \left\{\norm{M_\mu}_\infty, \frac{1}{\lambda}\norm{M_\Sigma}_\infty \right\}.
\end{equation}
Then, $T^\pi_2$ is a $\gamma$-contraction in $\norm{\,\cdot\,}_{\lambda}$.
\end{lemma}
We now consider the iterative process expressed by 
\begin{equation}\label{eq:jipe2eq}
M_{k+1} = T^\pi_2 M_k.
\end{equation}
\begin{theorem}\label{thm:jipe2}
$T^\pi_2$ admits a unique fixed point $M^\pi_2$. Moreover, for any initialization $M_0$, the iteration \eqref{eq:jipe2eq} has geometric convergence with rate $\gamma$ in $\norm{\,\cdot\,}_{\lambda}$:
\begin{equation}
\LVert M_k - M^\pi_2\RVert_{\lambda} \le \gamma^k\LVert M_0 - M^{\pi}_2 \RVert_{\lambda}.
\end{equation}
Furthermore, for any moment collection $M$, define the $2$nd-order Bellman residual as $\LVert M - T^\pi_2 M\RVert_{\lambda}$.
Then,
\begin{equation}
\LVert M - M^\pi_2 \RVert_{\lambda} \le \frac{1}{1-\gamma}\LVert M - T^\pi_2 M\RVert_{\lambda}.
\end{equation}
\end{theorem}
This enables the interpretation of \eqref{eq:jipe2eq} as an algorithm with a computable stopping criterion for $\epsilon$ accuracy in $\norm{\,\cdot\,}_{\lambda}$, which we call $2$nd-order joint iterative policy evaluation (JIPE-$2$).

\subsection{JIPE-2 with Function Approximation}
In high-dimensional or continuous state spaces, exact tabular JIPE-$2$ becomes impractical, so it is preferable to move to a function-approximation setting in which the moments are restricted to a parameterized $\hat{M}(\theta)$, with a PSD-structured parameterization for $M_\Sigma$ to preserve valid $2$nd moment geometry. The resulting objective is a projected fixed point equation of the form $\hat{M}(\theta^*) = \Pi T^\pi_2 \hat{M}(\theta^*)$. This yields a projected JIPE-$2$ iteration with a standard projection for the mean term and a PSD-constrained projection for the $2$nd-moment term. Under standard regularity conditions, Appendix \ref{app:funcapprox} shows that the projected operator is contractive in a suitable weighted norm, giving existence and uniqueness of the projected fixed point, convergence of the iteration, and a standard approximation-error bound.
\subsection{Incremental JIPE-2}

We next give an incremental variant of JIPE-$2$, suitable when $\mathcal{J}$ is accessed only through samples from its induced $1$- and $2$-JSTMs. The method is an instance of stochastic approximation for the operator $T^\pi_2$. Let us define an index set to simplify indexing a moment collection $M = (M_\mu, M_\Sigma)$:
\begin{equation}
\mathcal{I}_2 := \X \times \{ \mu\}\cup \X^2 \times \{\Sigma\}.
\end{equation}

For $i = ((s, a), \mu)$, we write $M(i) = M_\mu(s, a)$, and for $i = ((s, a), (\Tilde{s}, \Tilde{a}), \Sigma)$, we write $M(i) = M_\Sigma(s, a, \Tilde{s}, \Tilde{a})$.

We define the random one-sample backup $\hat{T}^\pi_2(M,i)$:
\begin{enumerate}[leftmargin=*, itemsep=0pt, align=left]
\item[(i)] If $i = ((s, a), \mu)$, we draw $(R^{(a)}, S'^{(a)}) \sim \mathcal{J}_1(\cdot \mid s, a)$ from the induced $1$-JSTM, $A'^{(a)} \sim \pi(\cdot \mid S'^{(a)})$. We set
\begin{equation}
\hat{T}^\pi_2 (M, i) := R^{(a)} + \gamma M_\mu\big(S'^{(a)}, A'^{(a)}\big).
\end{equation}
\item[(ii)] If $i = ((s, a),(s,a)), \Sigma)$, we draw $(R^{(a)}, S'^{(a)}) \sim \mathcal{J}_1(\cdot \mid s, a)$ from the induced $1$-JSTM and $
A'^{(a)} \sim \pi (\cdot \mid S'^{(a)})$. We set
\begin{equation}
\begin{gathered}
\hat{T}^\pi_2 (M, i) :=\big( R^{(a)}\big)^2+2\gamma R^{(a)}M_\mu\big(S'^{(a)}, A'^{(a)}\big) \\+ \gamma^2 M_\Sigma\big(S'^{(a)}, A'^{(a)}, S'^{(a)}, A'^{(a)}\big).
\end{gathered}
\end{equation}
\item[(iii)] If $i = ((s,a), (\Tilde{s}, \Tilde{a}), \Sigma)$, with $a \ne \Tilde{a}$ but $s= \Tilde{s}$, we draw $(R^{(a)}, S'^{(a)}, \Tilde{R}^{(\Tilde{a})}, \Tilde{S}'^{(\Tilde{a})}) \sim \mathcal{J}_2(\cdot \mid s, (a, \Tilde{a}))$ from the induced $2$-JSTM, $A'^{(a)} \sim \pi (\cdot \mid S'^{(a)})$, $  \Tilde{A}'^{(\Tilde{a})} \sim \pi(\cdot \mid \Tilde{S}'^{(\Tilde{a})})$. If $s \ne \Tilde{s}$, we independently draw $(R^{(a)}, S'^{(a)}) \sim \mathcal{J}_1(\cdot \mid s, a)$, $(\Tilde{R}^{(\Tilde{a})}, \Tilde{S}'^{(\Tilde{a})}) \sim \mathcal{J}_1(\cdot \mid \Tilde{s}, \Tilde{a})$ and $
A'^{(a)} \sim \pi (\cdot \mid S'^{(a)})$, $\Tilde{A}'^{(\Tilde{a})} \sim \pi(\cdot \mid \Tilde{S}'^{(\Tilde{a})})$. In either case, we set
\begin{equation}
\begin{gathered}
\hat{T}^\pi_2(M, i):=R^{(a)}\Tilde{R}^{(\Tilde{a})}\\+\gamma R^{(a)} M_\mu\big(\Tilde{S}'^{(\Tilde{a})},\Tilde{A}'^{(\Tilde{a})}\big)+  \gamma \Tilde{R}^{(\Tilde{a})} M_\mu\big(S'^{(a)},A'^{(a)}\big) \\+ \gamma^2 M_\Sigma\big(S'^{(a)}, A'^{(a)}, \Tilde{S}'^{(\Tilde{a})}, \Tilde{A}'^{(\Tilde{a})}\big).
\end{gathered}
\end{equation}
\end{enumerate}

For $\alpha_k \in (0, 1)$, we define the incremental JIPE-$2$ recursion:
\begin{equation}
\resizebox{.48\textwidth}{!}{$
M_{k+1}(i) = \begin{cases}
(1-\alpha_k(i))M_k(i)+\alpha_k(i)\hat{T}^\pi_2(M_k, i),&i=I_k,\\
M_k(i),&i \ne I_k,
\end{cases}
$}
\end{equation}
where $I_k$ designates the $k$th (random) index updated. The following theorem establishes almost-sure convergence.

\begin{theorem}\label{thm:jipe2-inc}
Assume each coordinate $i \in \mathcal{I}_2$ is selected infinitely often by $I_k$. Let the step size $\alpha_k(i)$ satisfy the conditions of \citet{robbins1951stochastic} for every index $i$:
\begin{equation}
\sum_k \alpha_k(i) = \infty,\quad \sum_k\alpha_k(i)^2 < \infty.
\end{equation} 
Then, the incremental JIPE-$2$ iteration converges almost surely in $\norm{\,\cdot\,}_\lambda$, i.e.,
\begin{equation}
\LVert M_k - M^\pi_2 \RVert_\lambda \to 0 \;\;\text{almost surely as} \; k \to \infty.
\end{equation}
\end{theorem}

\subsection{Generalization to higher moments}

It is logically straightforward to extend the construction to the case of $n$th-order moments, with heavier notation. For each $k \in [n]$, we consider a $k$th-order moment collection
\begin{equation}
M^{(k)}:= \X^k \to \R,\quad M^{(k)}(x_{1:k}) = \E \left[ \prod_{i=1}^k Z^\pi(x_i)\right].
\end{equation}

The collection of all moments of orders $1$ through $n$ is $M :=(M^{(1)},\ldots,M^{(n)}).$

Let $k \le n$ and $x_{1:k} \in \X^k$. Let $\{G_l\}_{l=1}^L$ be the partition of $[k]$ into groups that share the same current state, i.e., $i, j \in G_l$ for some $l$ if and only if $s_i = s_j$. For each group $G_l$ with common state $s^{(l)}$ and action subtuple $a_{G_l} := (a_i)_{i\in G_l}$, let
\begin{equation}
\big((R_i, S'_i) \big)_{i \in G_l} \sim \mathcal{J}_{\abs{G_l}}(\cdot \mid s^{(l)}, a_{G_l})
\end{equation}
under the induced $\abs{G_l}$-JSTM. Assume these groupwise draws are independent across $l$. Conditioned on $(S'_i)_{i=1}^k$, let actions be drawn independently as $A'_i \sim \pi(\cdot \mid S'_i)$ for $i \in [k]$. We write $X'_I := (X'_i)_{i \in I}$ for any index set $I \subseteq [k]$. This construction is the direct $k$-variate analogue of the joint laws used in the $2$nd-order scenario. Within a state, the relevant coordinates share one coupled draw via induced JSTMs, whereas across distinct states, the one-step randomness is independent under the one-step coupling regime.

We define the operator $T^\pi_n$ through the system of equations
\begin{equation}
\begin{gathered}
(T^\pi_n M)_k(x_{1:k}) \\:= \E \left[\sum_{I \subseteq [k]} \gamma^{\abs{I}} \big( \prod_{i\notin I} R_i\big) M_{\abs{I}}(X'_I) \mid x_{1:k} \right],
\end{gathered}
\end{equation}

for each $k \in [n]$ and each $x_{1:k} \in \X^k$. We also define
\begin{equation}
\LVert M\RVert_\lambda := \max_{k \in [n]}\frac{1}{\lambda_k} \LVert M^{(k)}\RVert_\infty,
\end{equation}
for $\lambda_k := (2/(1-\gamma))^{k-1}$. The operator $T^\pi_n$ is a $\gamma$-contraction with respect to $\norm{\,\cdot\,}_\lambda$, and admits as unique fixed point the true collection of moments up to $n$th order, $M^\pi_n$. The JIPE-$n$ algorithm and its incremental variant can be defined analogously to the previous section, and admit similar results to those of Theorems \ref{thm:jipe2} and \ref{thm:jipe2-inc}. These higher-order moments may be of use when mean and variance miss skewed or heavy-tailed action gap risks, analogous to high-order portfolio objective which extend mean-variance criteria with higher moments \citep{mvsk}.

\subsection{Connection to the gap RV}

We now return to the discussion of the gap RV $G^\pi(s; a, \Tilde{a})$. Quantities that only involve the first moments require no joint modeling, for instance, $\E[G^\pi(s; a, \Tilde{a})] = \mu^\pi(s, a) - \mu^\pi(s, \Tilde{a})$ is already determined by the marginal dynamics. However, joint structure becomes relevant when we consider comparative or risk-sensitive questions about gaps. JMDPs specify a concrete coupling of counterfactual outcomes for $(a, \Tilde{a})$ at $s$, inducing a well-defined joint return law. The moment recursions of the previous section provide mixed moments $\E[Z^\pi(s, a)^i Z^\pi(s,\Tilde{a})^j]$, which determine the moments of $G^\pi(s; a, \Tilde{a})$ and enable moment-based bounds or approximations for comparative risk criteria. For example, $2$nd-order mixed moments yield the variance of the gap:
\begin{equation}
\begin{gathered}
\mathrm{var}\big(G^\pi(s;a,\Tilde{a})\big) = \Sigma^\pi(s,a,s,a) - \Sigma^\pi(s,\Tilde{a},s,\Tilde{a}) \\- 2\Sigma^\pi(s,a,s,\Tilde{a}) - \big(\mu^\pi(s, a) - \mu^\pi(s, \Tilde{a}) \big)^2,
\end{gathered}
\end{equation}
where all moments are computable through JIPE-$2$.

Similarly, if $\mu_G := \E[G^\pi(s; a, \Tilde{a})>\allowbreak 0]$, by the Chebyshev inequality \citep{boucheron2013} $\mathbb{P}(G^\pi(s; a, \Tilde{a}) \le\allowbreak 0 ) \le \allowbreak\sigma^2_G/(\sigma^2_G + \mu^2_G)$, where $\sigma^2_G \allowbreak := \mathrm{var}(G^\pi(s; a, \Tilde{a}))$.
\section{Experiments}

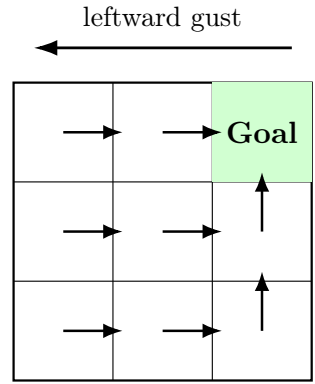
\begin{wrapfigure}{r}{0.15\textwidth}
\vspace{-5mm}   \centering\resizebox{0.15\textwidth}{!}{\begin{tikzpicture}[scale=1.25, every node/.style={font=\small}]
  \tikzset{pol/.style={-Latex, line width=0.9pt}}
  \tikzset{gust/.style={-{Latex[length=3mm,width=2mm]}, line width=1pt}}

  \draw[thick] (0,0) rectangle (3,3);
  \foreach \x in {1,2} \draw ( \x,0 ) -- ( \x,3 );
  \foreach \y in {1,2} \draw ( 0,\y ) -- ( 3,\y );

  \fill[green!18] (2,2) rectangle (3,3);
  \node[font=\bfseries] at (2.5,2.5) {Goal};

  \draw[pol] (0.5,0.5) -- ++(0.60,0); 
  \draw[pol] (1.5,0.5) -- ++(0.60,0); 
  \draw[pol] (2.5,0.5) -- ++(0,0.60); 
  \draw[pol] (0.5,1.5) -- ++(0.60,0); 
  \draw[pol] (1.5,1.5) -- ++(0.60,0);  
  \draw[pol] (2.5,1.5) -- ++(0,0.60); 
  \draw[pol] (0.5,2.5) -- ++(0.60,0); 
  \draw[pol] (1.5,2.5) -- ++(0.60,0); 

  \draw[gust] (2.8,3.35) -- (0.2,3.35);
  \node[above=2pt] at (1.5,3.35) {leftward gust};

\end{tikzpicture}}
    \caption{A $3\times3$ WGW environment and the evaluated policy.}
    \vspace{-5mm}
    \label{wgwpolicy}
\label{fig:wgw}
\end{wrapfigure}
We empirically validate the policy-evaluation theory developed in the previous sections in four complementary ways. First, in tabular coupled-dynamics environments we run JIPE-$2$ and track the Bellman residual $\norm{M - T^\pi_2 M}_\lambda$, which provides a certificate of accuracy. Second, we visualize the learned joint structure via return correlation matrices across actions, computed from the first and second moments. Third, we demonstrate that the learned mixed moments enable meaningful estimates and bounds for the gap RV, and validate these estimates against MC simulation. Finally, we show that the incremental JIPE-$2$ with neural function approximation scales beyond tabular settings in coupled ALE environments \cite{ALE}.

All experiments are policy evaluation: We fix $\pi$ and estimate the first moments $\mu^\pi(s,a)$ and second moments $\Sigma^\pi(s, a, \Tilde{s}, \Tilde{a})$. In tabular experiments, we apply the exact operator $T^\pi_2$ and report the Bellman residual. In large-scale experiments, we implement incremental JIPE-$2$ with neural approximation, and report moving-average TD errors.

\begin{figure}
  \centering
  \begin{minipage}{0.24\textwidth}
    \centering
    \includegraphics[width=\linewidth]{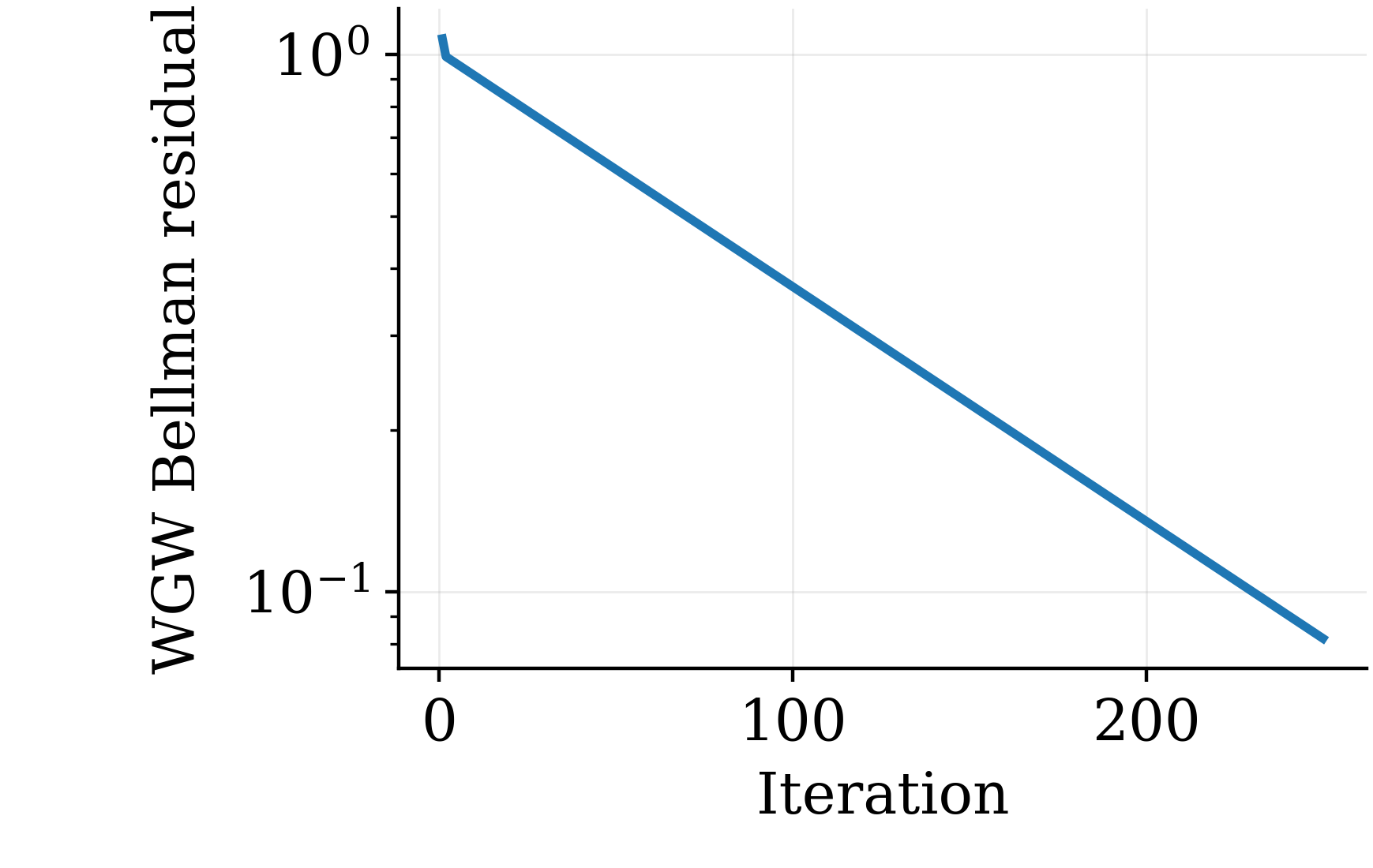}
  \end{minipage}\hfill
  \begin{minipage}{0.24\textwidth}
    \centering
    \includegraphics[width=\linewidth]{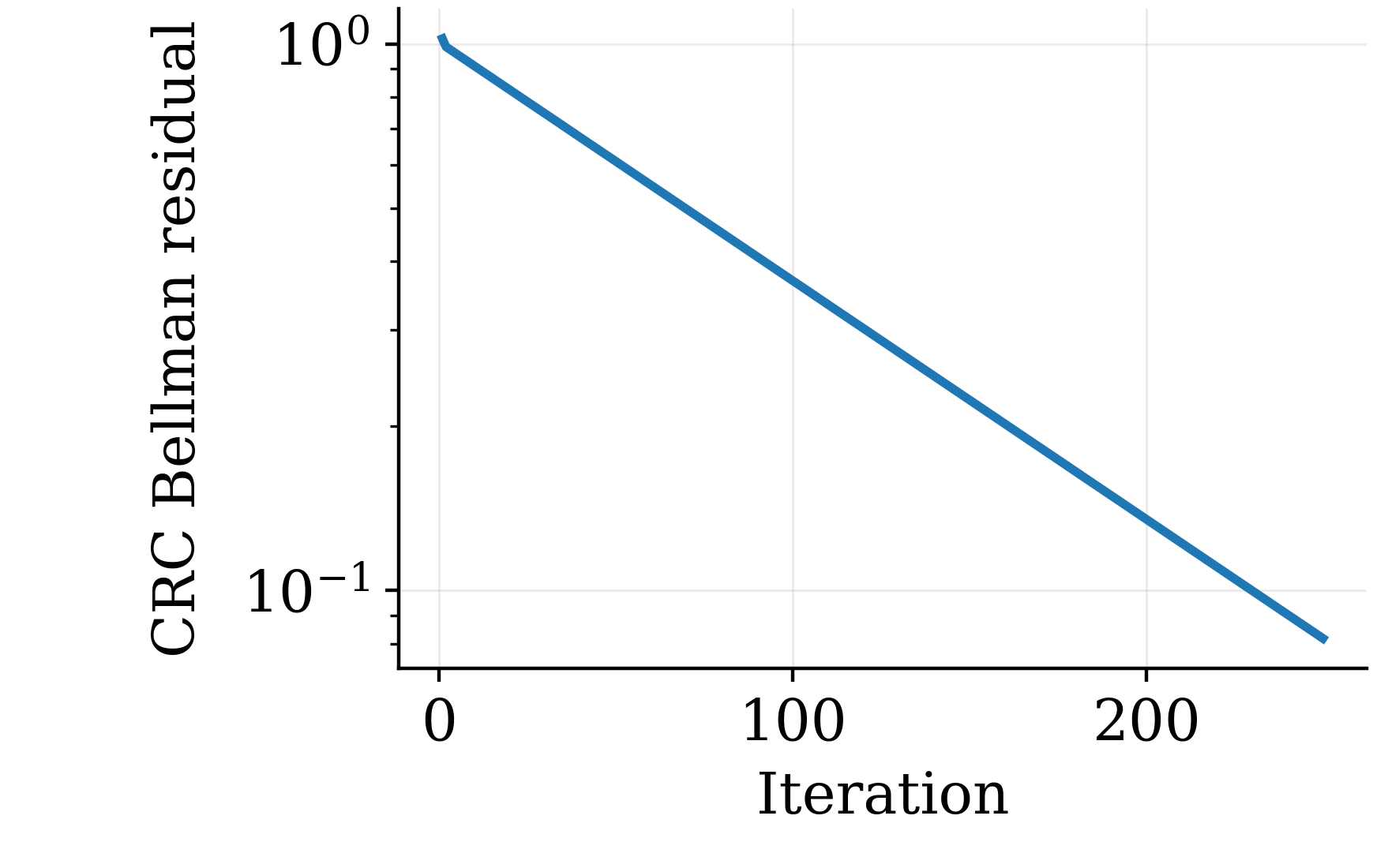}
  \end{minipage}
  \caption{\textbf{Bellman residual convergence in tabular coupled-dynamics environments.}
  \textbf{Left:} $5\times 5$ WGW. \textbf{Right:} CRC with $\lvert \Ss\rvert=25$.
  We plot the Bellman residual $\lVert M_k - T_2^\pi M_k\rVert_\lambda$ on logarithmic scale.}
  \label{fig:residuals_tabular}
  \vspace{-3mm}
\end{figure}

\textbf{Windy gridworld (WGW).} We consider a WGW environment with actions $\{\textsf{U}, \textsf{R}, \textsf{D}, \textsf{L}\}$  and a leftward ``wind gust'' that couples counterfactual next states under shared exogenous randomness (Figure \ref{fig:wgw}). Figure \ref{fig:residuals_tabular} (left) shows linear decay of the Bellman residual on a log scale, consistent with the $\gamma$-contraction and geometric convergence results.

\begin{wrapfigure}{r}{0.2\textwidth}
\centering\resizebox{0.2\textwidth}{!}{\begin{tikzpicture}[every node/.style={font=\small}]
  \node[circle, draw, minimum size=0.62cm] (si) at (0,0) {$s_i$};
  \node[circle, draw, minimum size=0.62cm] (sj) at (3.2,0) {$s_{i+1}$};
  \draw[->, very thick, blue!65!black, bend left=32]
    (si) to node[above] {$a_1,\ +r_i$} (sj);
  \draw[->, very thick, red!70!black, bend right=32]
    (si) to node[below] {$a_2,\ -r_i$} (sj);
\end{tikzpicture}}
    \caption{One CRC transition. Both actions reach the same next state but receive anti-correlated rewards.}
    \vspace{-5mm}
\label{fig:crc}
\end{wrapfigure}
\textbf{Coupled-reward chain (CRC).} We consider a finite chain of states with two actions and anti-correlated rewards (Figure \ref{fig:crc}). Figure \ref{fig:residuals_tabular} (right) again displays linear decay of the Bellman residual on a log scale. To visualize the cross-action dependence, we form, for each state $s$, the covariance and correlation matrices from the estimated uncentered second moments. Figure \ref{fig:wgw_corr_grid} shows, for a $3 \times 3$ WGW under a fixed goal-directed policy, the $4\times 4$ correlation matrix $\rho^\pi_s$ at each state. The figure indicates that the coupled-dynamics induce a structured, state-dependent joint law across actions which is invisible to an MDP marginal description. Figure \ref{fig:chain_corr} similarly reports the learned return correlation matrix $\rho^\pi_{s_0}$ at the initial chain state $s_0$ of a CRC with $\lvert \Ss \rvert =25$.

\begin{figure}
  \centering
  \includegraphics[width=0.25\textwidth]{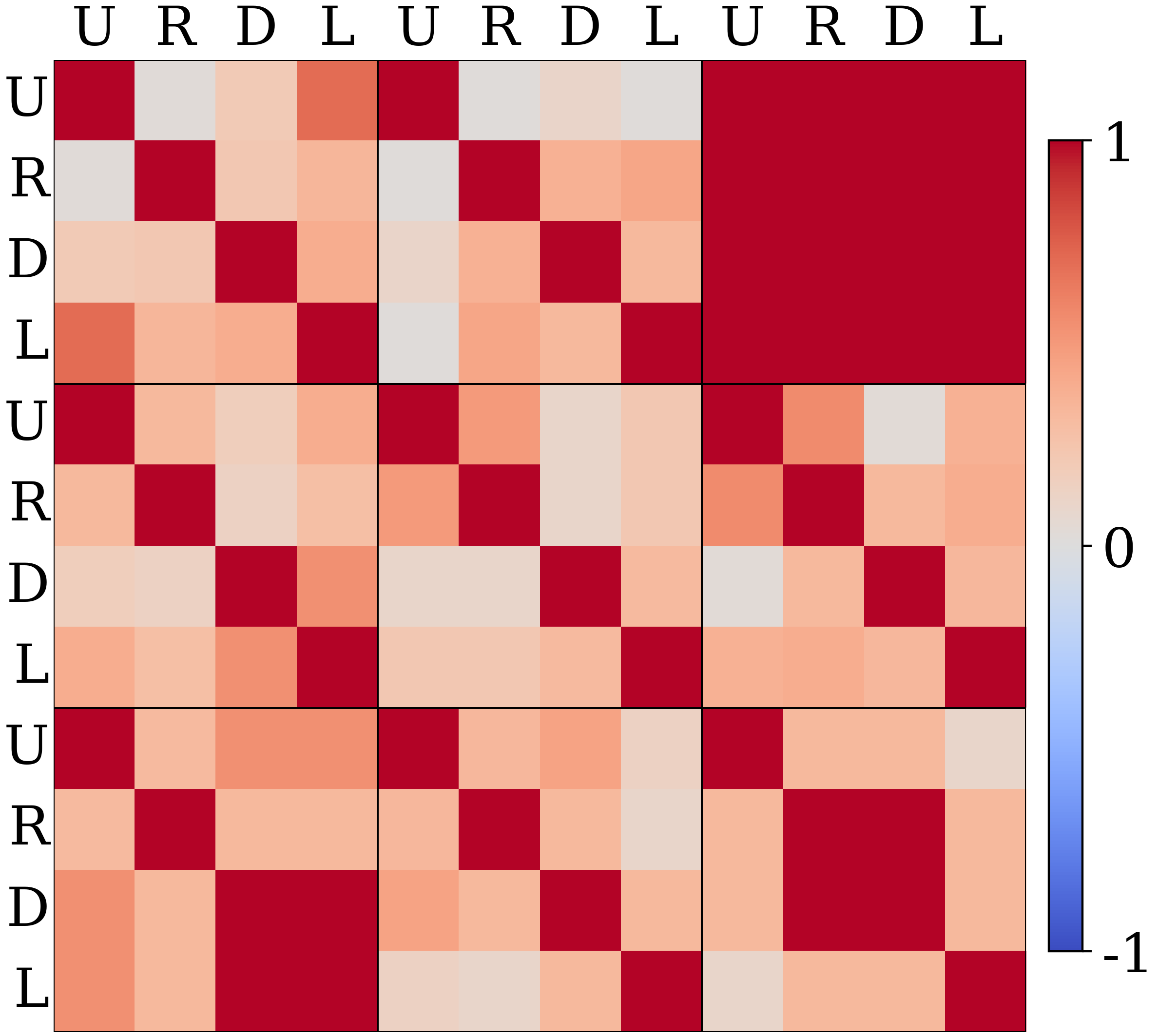}
  \caption{\textbf{Per-state action correlation matrices in WGW.}
  Each tile corresponds to a state in a $3\times 3$ WGW, and displays the $4\times 4$ correlation matrix
  $\rho^\pi_s$ computed through JIPE-$2$.}
  \label{fig:wgw_corr_grid}
  \vspace{-3mm}
\end{figure}

We then validate the use of mixed moments for gap statistics. Figure \ref{fig:gap_validation} (left, middle) compare JIPE-$2$-derived predictions of $\E[G^\pi]$ and $\mathrm{var}(G^\pi)$ to MC estimates. The agreement indicates that the learned mixed moments meaningfully capture joint-action return structure. Furthermore, using the Chebyshev inequality, we derive an upper bound on the inferiority probability $\mathbb{P}(G^\pi \le 0)$. We summarize the bound's tightness by the empirical cdf (ECDF) of the ratio $\hat{\mathbb{P}}(G^\pi \le 0) / \mathrm{Chebyshev(G^\pi)}$ (where $\hat{\mathbb{P}}$ is an MC estimate). Figure \ref{fig:gap_validation} (right) shows ratios bounded by $1$, indicating empirical non-violation of the upper-bound.
\begin{wrapfigure}{l}{0.175\textwidth}
\centering\resizebox{0.175\textwidth}{!}{\includegraphics{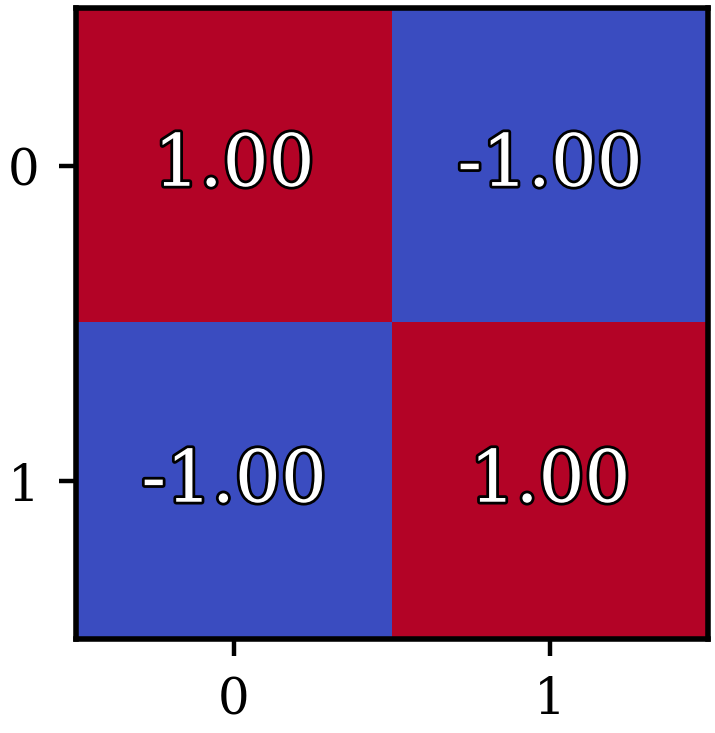}}
    \caption{Correlation matrix $\rho^\pi_{s_0}$ computed from the converged JIPE-$2$ moments in CRC.}
  \label{fig:chain_corr}
    \vspace{-2mm}
\end{wrapfigure}
To demonstrate scalability beyond tabular settings, we implement incremental JIPE-$2$ with neural function approximation in four ALE environments with a coupled-dynamics interface. Concretely, we wrap the simulator with a multi-action one-step query mechanism that returns counterfactual one-step outcomes for multiple actions under shared emulator randomness at the queried state. We parameterize $M_\mu(s,a)$ and $M_\Sigma(s, a, \Tilde{s}, \Tilde{a})$ with neural networks and perform SGD on TD errors. Figure \ref{fig:atari_td} reports the moving-average TD errors for the $\mu$ coordinate and multiple $\Sigma$ coordinate blocks. 
Across all cases, TD errors decrease by several orders of magnitude, providing evidence that incremental JIPE-$2$ can be combined with function approximation to mitigate the $\lvert \Ss \rvert^2\lvert \A\rvert^2$ complexity of tabular $2$nd moment evaluation. Appendix~\ref{app:ale} provides implementation and validation details. 

\section{Conclusion}
We introduced JMDPs for coupled-dynamics environments that explicitly model the joint law of counterfactual outcomes across actions. Under a one-step coupling regime, we derived Bellman operators for return moments and obtained convergence guarantees for DP algorithms and their incremental variants, along with a projected function-approximation formulation. This makes intrinsically joint quantities such as gap moments and inferiority-probability bounds computable from multi-action  simulator access. Experiments in coupled tabular and ALE environments show that learned mixed moments recover meaningful cross-action dependence absent from marginal MDP descriptions. A natural next step is control: extending joint moment evaluation to policy improvement under joint distributional objectives in JMDPs.

\begin{figure*}
  \centering
  \begin{minipage}{0.32\textwidth}
    \centering
    \includegraphics[width=\linewidth]{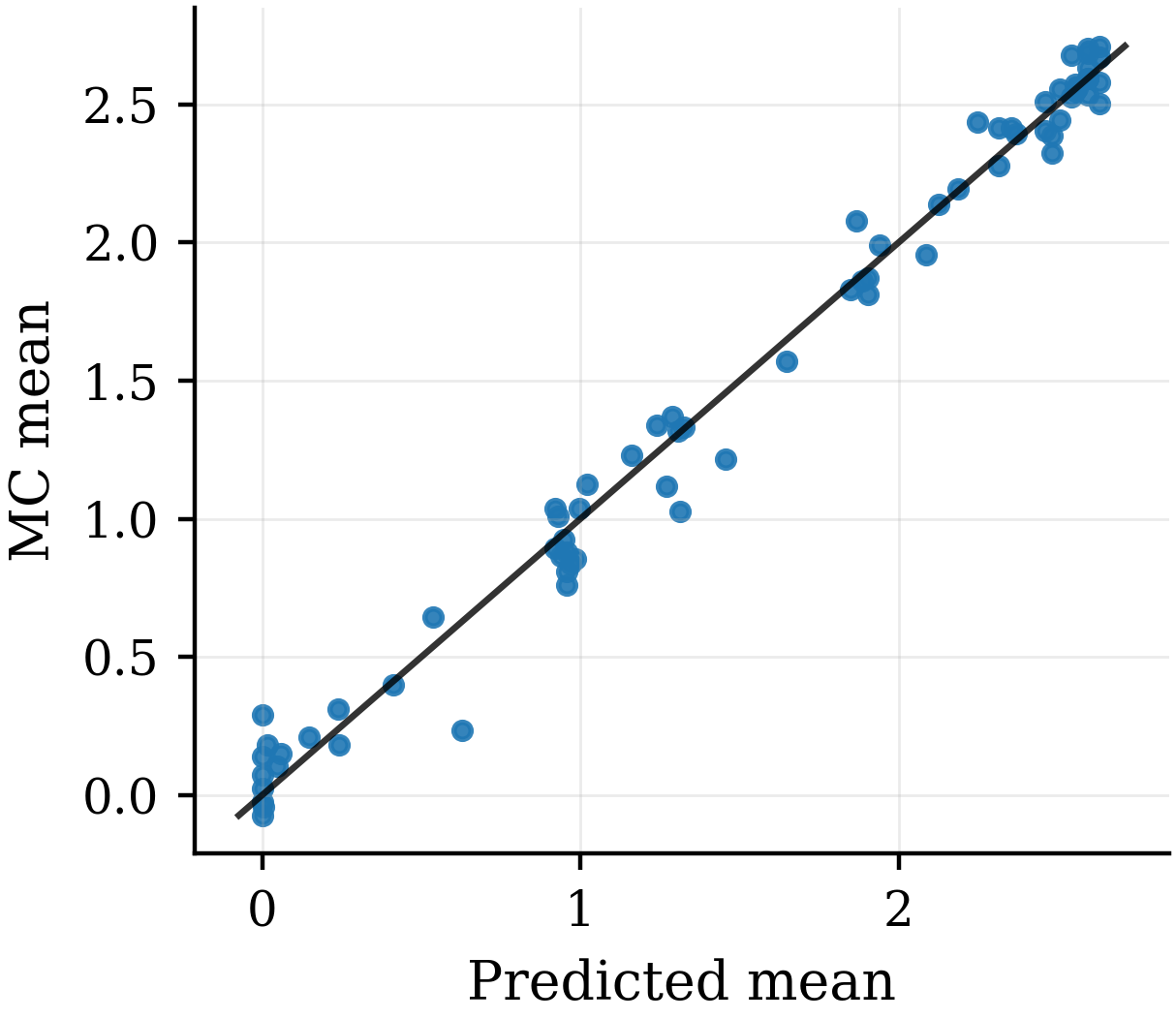}
  \end{minipage}\hfill
  \begin{minipage}{0.32\textwidth}
    \centering
    \includegraphics[width=\linewidth]{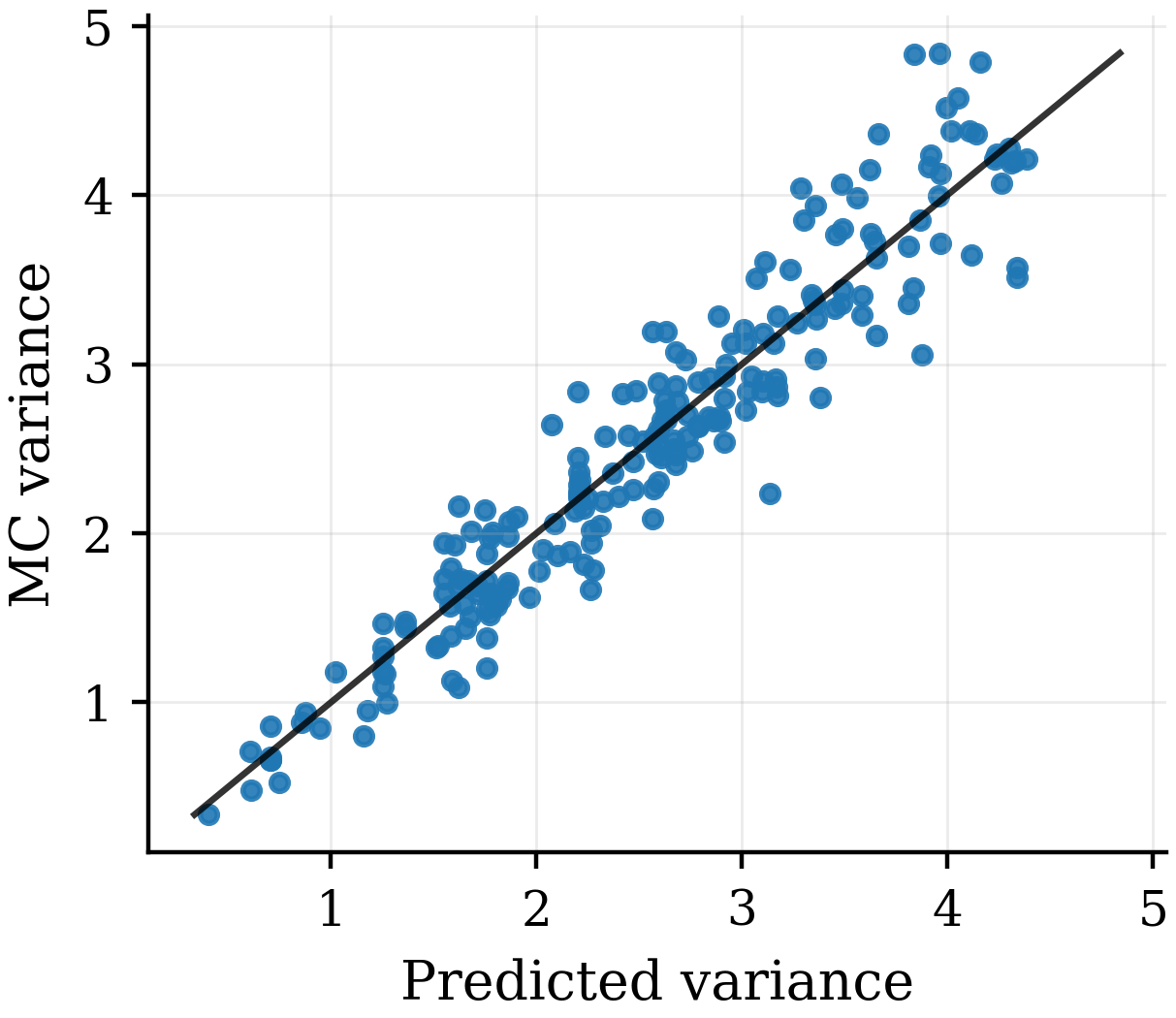}
  \end{minipage}\hfill
  \begin{minipage}{0.32\textwidth}
    \centering
    \includegraphics[width=\linewidth]{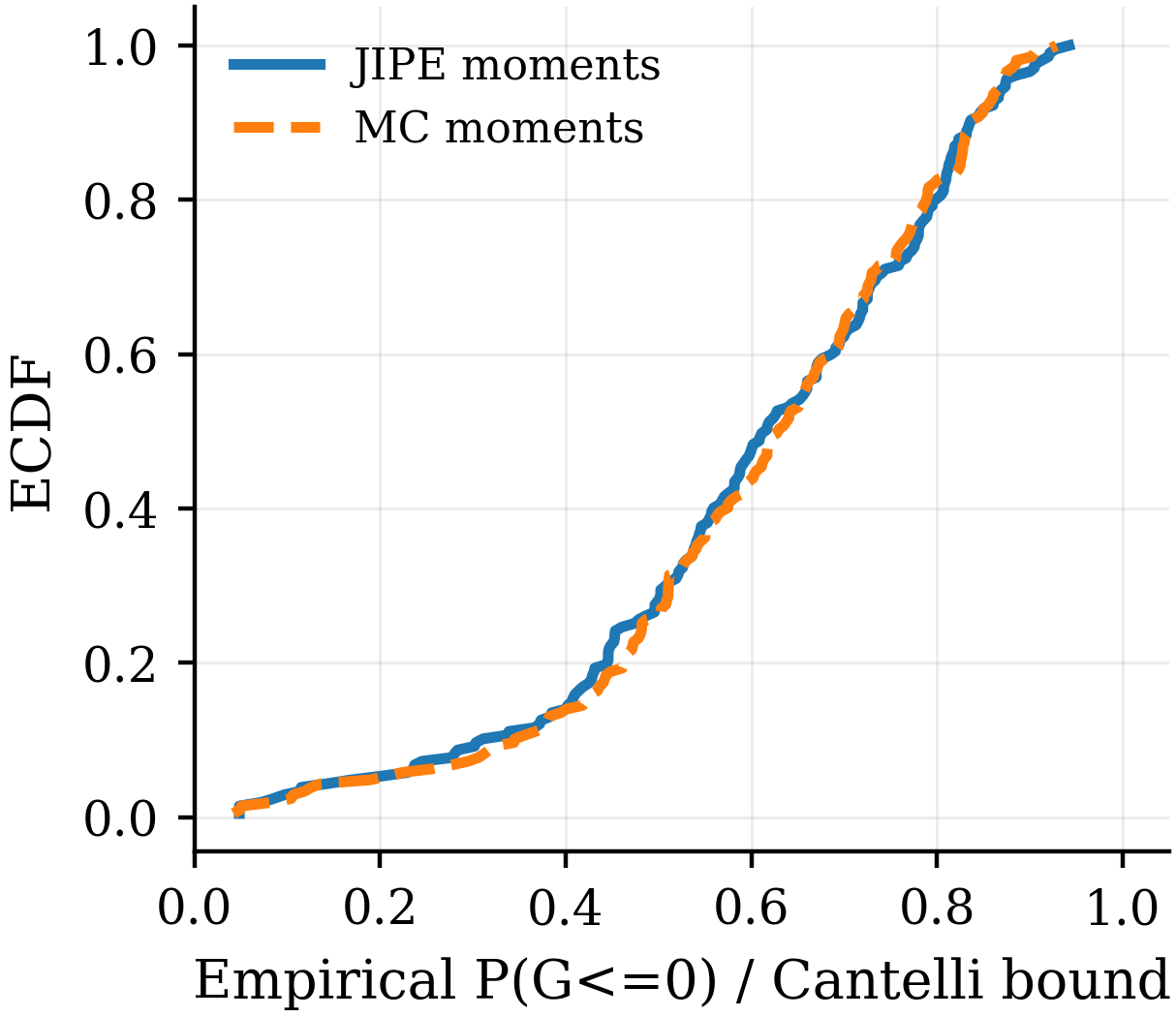}
  \end{minipage}
  \caption{\textbf{Gap validation in WGW.}
  We evaluate a fixed goal-directed policy and consider gaps between the policy’s action and alternatives.
  \textbf{Left:} Predicted vs. MC gap means $\E[G^\pi]$.
  \textbf{Middle:} Predicted vs. MC gap variances $\mathrm{var}(G^\pi)$.
  \textbf{Right:} Empirical cdf of the ratio $\hat{\mathbb{P}}(G^\pi\le 0)/\mathrm{Chebyshev}(G^\pi)$. Ratios closer to $1$ indicate tighter upper bounds, while small ratios indicate looseness. The \textbf{blue} curve computes the denominator using JIPE-$2$-estimated moments, while the \textbf{orange} curve uses MC-estimated moments (a near-ground-truth proxy), separating intrinsic bound looseness from moment-estimation error.}
\label{fig:gap_validation}
\vspace{-3mm}
\end{figure*}

\begin{figure*}
  \centering
  \begin{minipage}{0.49\textwidth}
    \centering
    \includegraphics[width=\linewidth]{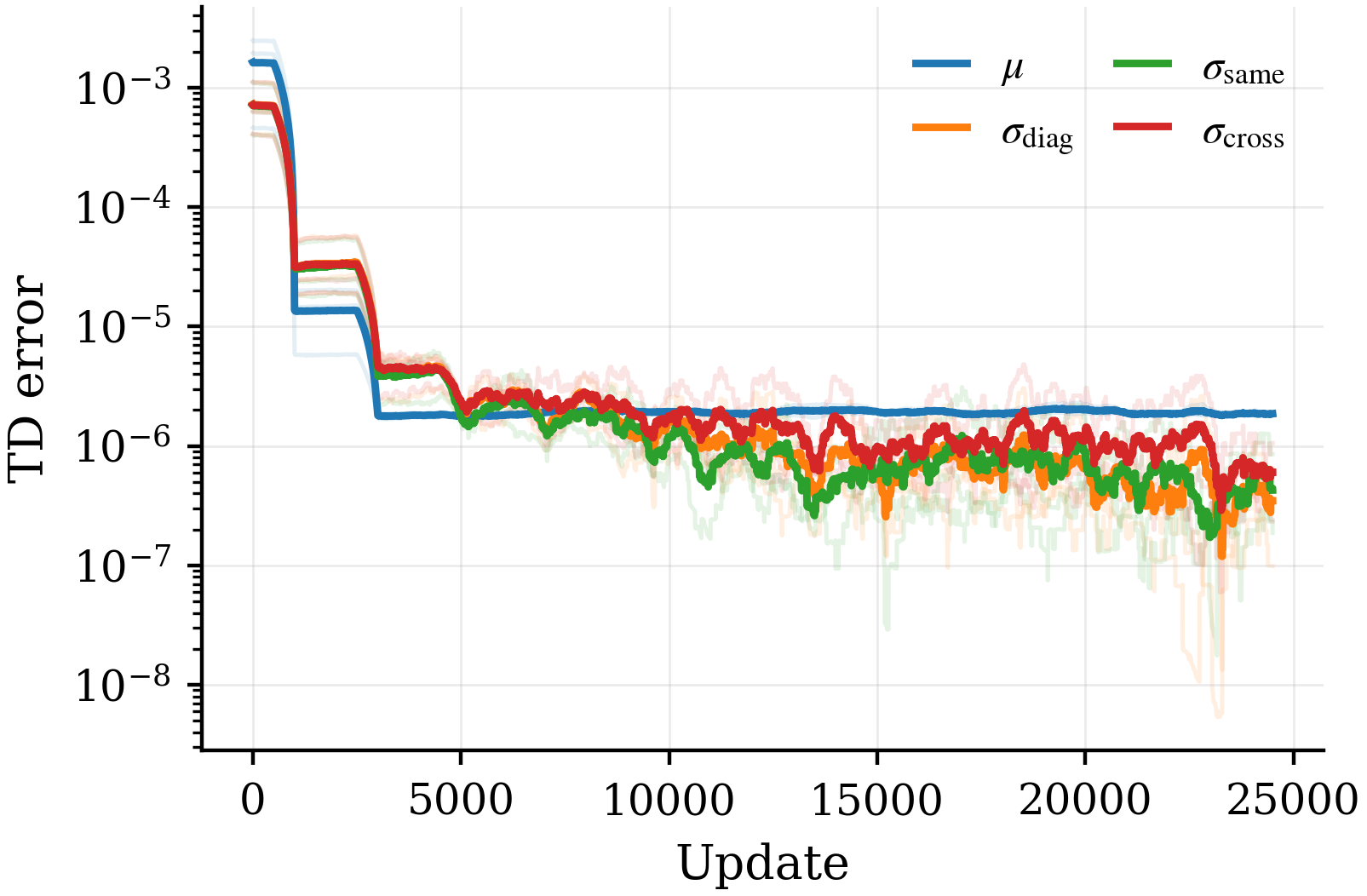}
  \end{minipage}\hfill
  \begin{minipage}{0.49\textwidth}
    \centering
    \includegraphics[width=\linewidth]{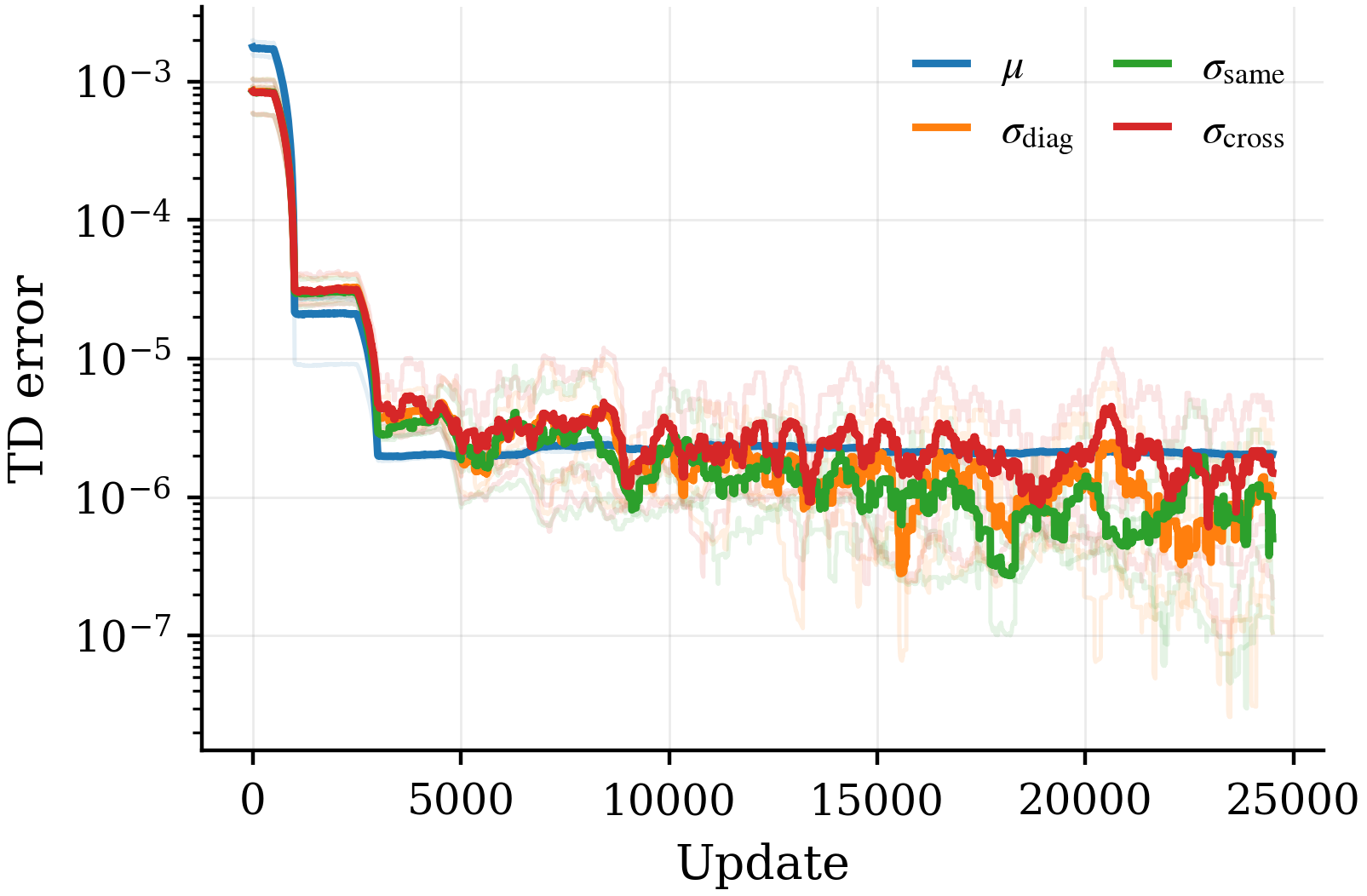}
  \end{minipage}

  \vspace{1.5mm}

  \begin{minipage}{0.49\textwidth}
    \centering
    \includegraphics[width=\linewidth]{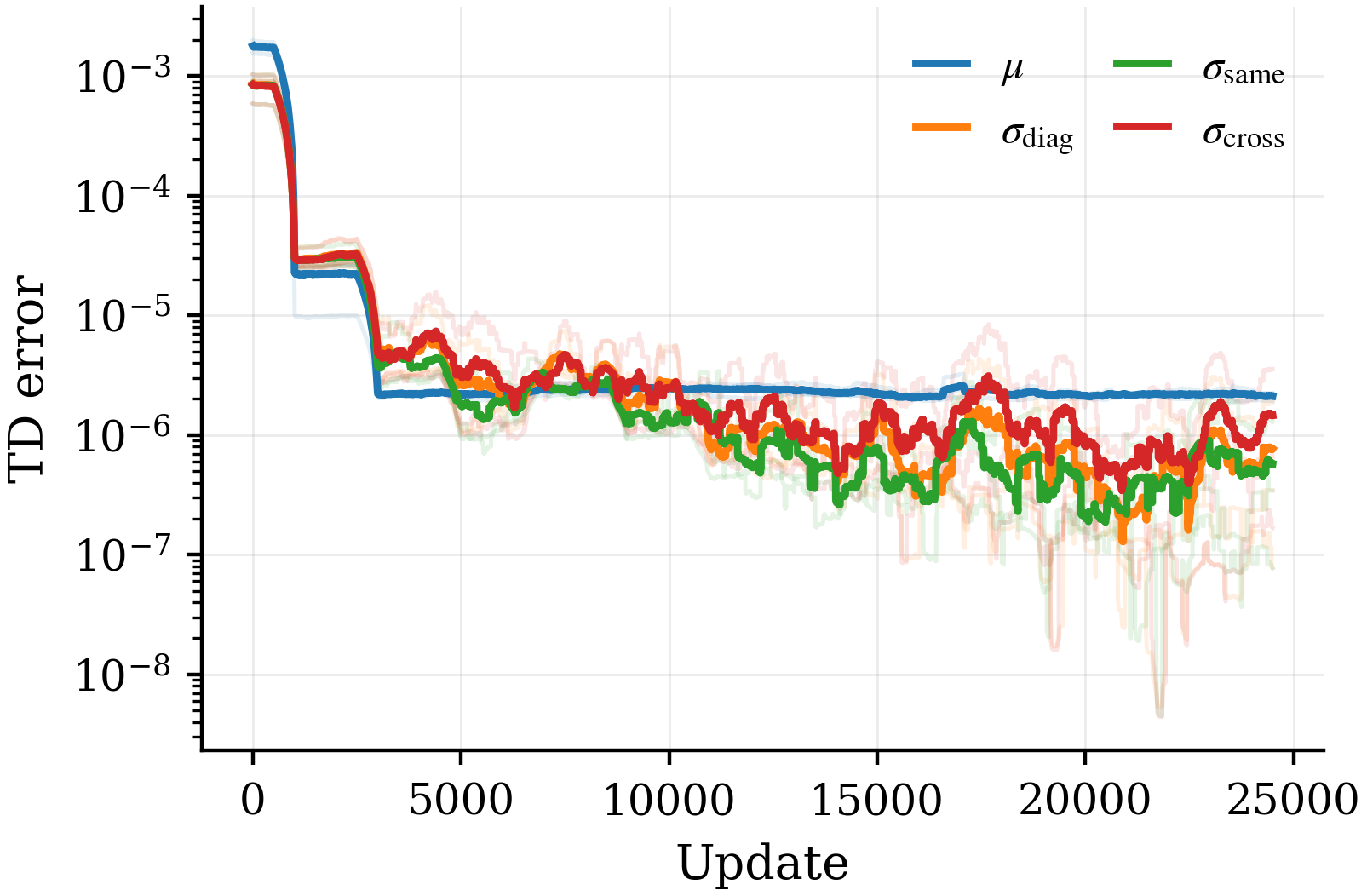}
  \end{minipage}\hfill
  \begin{minipage}{0.49\textwidth}
    \centering
    \includegraphics[width=\linewidth]{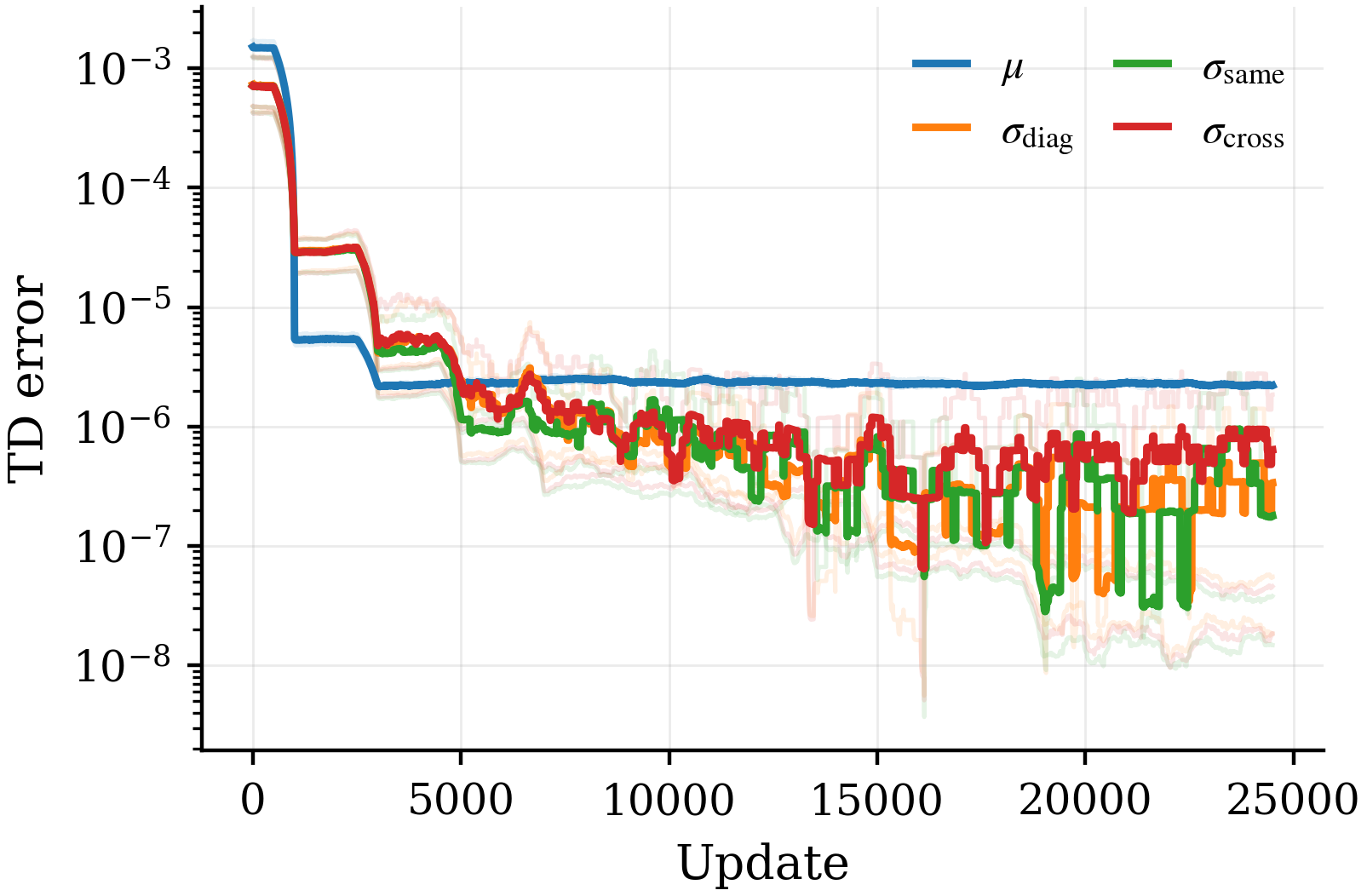}
  \end{minipage}
  \caption{\textbf{Incremental JIPE-$2$ with neural functional approximation on coupled ALE environments.}
  Moving-average TD errors (log scale) during SGD-based incremental JIPE-$2$ with neural function approximation, shown for Pong, BattleZone, Boxing, and Atlantis (left-to-right, top-to-bottom). The plotted coordinates corresponds to the JIPE-$2$ blocks: $\mu$ is the $1$st-moment (mean-return) component, $\sigma_{\mathrm{diag}}$ is the block where $s=\Tilde{s}$ and $a=\Tilde{a}$. $\sigma_{\mathrm{same}}$ is the block where $s =\Tilde{s}$, $a \ne \Tilde{a}$. $\sigma_{\mathrm{cross}}$ is the block where $s \ne \Tilde{s}$, $a \ne \Tilde{a}$. Bold curves show the across-seed mean TD error, and faint curves show the individual trajectories from \textbf{three random seeds}. Across all four ALE games, each block exhibits stable descent over training, supporting the practicality of the incremental JIPE-$2$ updates under nonlinear function approximation. Gap moments are evaluated separately against recursive-pair MC in Appendix \ref{app:ale}.}
  \label{fig:atari_td}
  \vspace{-3mm}
\end{figure*}

\newpage

\begin{acknowledgements}
This work was supported in part by NSF under grants DMS-2502560 and CNS-2313109.
\end{acknowledgements}
\bibliography{uai2026-template}

@InProceedings{c51,
  title = 	 {A Distributional Perspective on Reinforcement Learning},
  author =       {Marc G. Bellemare and Will Dabney and R{\'e}mi Munos},
  booktitle = 	 {Proceedings of the 34th International Conference on Machine Learning},
  pages = 	 {449--458},
  year = 	 {2017},
  editor = 	 {Precup, Doina and Teh, Yee Whye},
  volume = 	 {70},
  series = 	 {Proceedings of Machine Learning Research},
  month = 	 {06--11 Aug},
  publisher =    {PMLR},
  url = 	 {https://proceedings.mlr.press/v70/bellemare17a.html}
}

@book{BartoSutton,
author = {Sutton, Richard S. and Barto, Andrew G.},
title = {Reinforcement Learning: An Introduction},
year = {2018},
isbn = {0262039249},
publisher = {A Bradford Book},
address = {Cambridge, MA, USA}
}

@inproceedings{cvar-mdp,
author = {Chow, Yinlam and Ghavamzadeh, Mohammad},
title = {Algorithms for CVaR optimization in MDPs},
year = {2014},
publisher = {MIT Press},
address = {Cambridge, MA, USA},
booktitle = {Proceedings of the 28th International Conference on Neural Information Processing Systems - Volume 2},
pages = {3509–3517},
numpages = {9},
location = {Montreal, Canada},
series = {NIPS'14}
}

@book{DRL-textbook,
	author = {Bellemare, Marc G. and Dabney, Will and Rowland, Mark},
	doi = {10.7551/mitpress/14207.001.0001},
	eprint = {https://direct.mit.edu/book-pdf/2111075/book\_9780262374026.pdf},
	isbn = {9780262374026},
	month = {05},
	publisher = {The MIT Press},
	title = {Distributional Reinforcement Learning},
	url = {https://doi.org/10.7551/mitpress/14207.001.0001},
	year = {2023}}

@article{estimating-variance,
  author  = {Aviv Tamar and Dotan Di Castro and Shie Mannor},
  title   = {Learning the Variance of the Reward-To-Go},
  journal = {Journal of Machine Learning Research},
  year    = {2016},
  volume  = {17},
  number  = {13},
  pages   = {1--36},
  url     = {http://jmlr.org/papers/v17/14-335.html}
}

@article{FQF,
  title={Fully parameterized quantile function for distributional reinforcement learning},
  author={Yang, Derek and Zhao, Li and Lin, Zichuan and Qin, Tao and Bian, Jiang and Liu, Tie-Yan},
  journal={Advances in neural information processing systems},
  volume={32},
  year={2019}
}

@InProceedings{iqn,
  title = 	 {Implicit Quantile Networks for Distributional Reinforcement Learning},
  author =       {Dabney, Will and Ostrovski, Georg and Silver, David and Munos, Remi},
  booktitle = 	 {Proceedings of the 35th International Conference on Machine Learning},
  pages = 	 {1096--1105},
  year = 	 {2018},
  editor = 	 {Dy, Jennifer and Krause, Andreas},
  volume = 	 {80},
  series = 	 {Proceedings of Machine Learning Research},
  month = 	 {10--15 Jul},
  publisher =    {PMLR},
  url = 	 {https://proceedings.mlr.press/v80/dabney18a.html}
}

@inproceedings{mv-mdp,
author = {Mannor, Shie and Tsitsiklis, John N.},
title = {Mean-variance optimization in Markov decision processes},
year = {2011},
isbn = {9781450306195},
publisher = {Omnipress},
address = {Madison, WI, USA},
booktitle = {Proceedings of the 28th International Conference on International Conference on Machine Learning},
pages = {177–184},
numpages = {8},
location = {Bellevue, Washington, USA},
series = {ICML'11}
}

@ARTICLE{mvsk,
  author={Zhou, Rui and Palomar, Daniel P.},
  journal={IEEE Transactions on Signal Processing}, 
  title={Solving High-Order Portfolios via Successive Convex Approximation Algorithms}, 
  year={2021},
  volume={69},
  number={},
  pages={892-904},
  doi={10.1109/TSP.2021.3051369}}

@inproceedings{qr-dqn,
author = {Dabney, Will and Rowland, Mark and Bellemare, Marc G. and Munos, R\'{e}mi},
title = {Distributional reinforcement learning with quantile regression},
year = {2018},
isbn = {978-1-57735-800-8},
publisher = {AAAI Press},
booktitle = {Proceedings of the Thirty-Second AAAI Conference on Artificial Intelligence and Thirtieth Innovative Applications of Artificial Intelligence Conference and Eighth AAAI Symposium on Educational Advances in Artificial Intelligence},
articleno = {353},
numpages = {10},
location = {New Orleans, Louisiana, USA},
series = {AAAI'18/IAAI'18/EAAI'18}
}

@article{ALE,
author = {Bellemare, Marc and Naddaf, Yavar and Veness, Joel and Bowling, Michael},
year = {2012},
month = {07},
pages = {},
title = {The Arcade Learning Environment: An Evaluation Platform for General Agents},
volume = {47},
journal = {Journal of Artificial Intelligence Research},
doi = {10.1613/jair.3912}
}

@article{sobel,
 ISSN = {00219002},
 URL = {http://www.jstor.org/stable/3213832},
 author = {Matthew J. Sobel},
 journal = {Journal of Applied Probability},
 number = {4},
 pages = {794--802},
 publisher = {Applied Probability Trust},
 title = {The Variance of Discounted Markov Decision Processes},
 urldate = {2025-04-15},
 volume = {19},
 year = {1982}
}

@InProceedings{tamar2,
  title = 	 {Temporal Difference Methods for the Variance of the Reward To Go},
  author = 	 {Tamar, Aviv and Di Castro, Dotan and Mannor, Shie},
  booktitle = 	 {Proceedings of the 30th International Conference on Machine Learning},
  pages = 	 {495--503},
  year = 	 {2013},
  editor = 	 {Dasgupta, Sanjoy and McAllester, David},
  volume = 	 {28},
  series = 	 {Proceedings of Machine Learning Research},
  address = 	 {Atlanta, Georgia, USA},
  month = 	 {17--19 Jun},
  publisher =    {PMLR},
  url = 	 {https://proceedings.mlr.press/v28/tamar13.html}
}

@book{puterman, author = {Puterman, Martin L.}, title = {Markov Decision Processes: Discrete Stochastic Dynamic Programming}, year = {1994}, isbn = {0471619779}, publisher = {John Wiley \& Sons, Inc.}, address = {USA}, edition = {1st} }

@InProceedings{bellman-gan,
  title = 	 {Distributional Multivariate Policy Evaluation and Exploration with the {B}ellman {GAN}},
  author =       {Freirich, Dror and Shimkin, Tzahi and Meir, Ron and Tamar, Aviv},
  booktitle = 	 {Proceedings of the 36th International Conference on Machine Learning},
  pages = 	 {1983--1992},
  year = 	 {2019},
  editor = 	 {Chaudhuri, Kamalika and Salakhutdinov, Ruslan},
  volume = 	 {97},
  series = 	 {Proceedings of Machine Learning Research},
  month = 	 {09--15 Jun},
  publisher =    {PMLR},
  url = 	 {https://proceedings.mlr.press/v97/freirich19a.html}
}

@article{lu2020sample,
  title={Sample-efficient reinforcement learning via counterfactual-based data augmentation},
  author={Lu, Chaochao and Huang, Biwei and Wang, Ke and Hern{\'a}ndez-Lobato, Jos{\'e} Miguel and Zhang, Kun and Sch{\"o}lkopf, Bernhard},
  journal={arXiv preprint arXiv:2012.09092},
  year={2020}
}

@inproceedings{amitai2024explaining,
  title={Explaining reinforcement learning agents through counterfactual action outcomes},
  author={Amitai, Yotam and Septon, Yael and Amir, Ofra},
  booktitle={Proceedings of the AAAI Conference on Artificial Intelligence},
  volume={38},
  issue={9},
  pages={10003--10011},
  year={2024}
}

@article{wiltzer2024action,
  title={Action gaps and advantages in continuous-time distributional reinforcement learning},
  author={Wiltzer, Harley and Bellemare, Marc and Meger, David and Shafto, Patrick and Jhaveri, Yash},
  journal={Advances in Neural Information Processing Systems},
  volume={37},
  pages={47815--47848},
  year={2024}
}

@book{glasserman2004monte,
  title={Monte Carlo methods in financial engineering},
  author={Glasserman, Paul},
  volume={53},
  year={2004},
  publisher={Springer}
}

@book{shapiro2021lectures,
  title={Lectures on stochastic programming: modeling and theory},
  author={Shapiro, Alexander and Dentcheva, Darinka and Ruszczynski, Andrzej},
  year={2021},
  publisher={SIAM}
}

@inproceedings{nguyen2021distributional,
  title={Distributional reinforcement learning via moment matching},
  author={Nguyen-Tang, Thanh and Gupta, Sunil and Venkatesh, Svetha},
  booktitle={Proceedings of the AAAI conference on artificial intelligence},
  volume={35},
  pages={9144--9152},
  year={2021}
}

@article{zhang2021distributional,
  title={Distributional reinforcement learning for multi-dimensional reward functions},
  author={Zhang, Pushi and Chen, Xiaoyu and Zhao, Li and Xiong, Wei and Qin, Tao and Liu, Tie-Yan},
  journal={Advances in Neural Information Processing Systems},
  volume={34},
  pages={1519--1529},
  year={2021}
}

@article{wiltzer2024foundations,
  title={Foundations of multivariate distributional reinforcement learning},
  author={Wiltzer, Harley and Farebrother, Jesse and Gretton, Arthur and Rowland, Mark},
  journal={Advances in Neural Information Processing Systems},
  volume={37},
  pages={101297--101336},
  year={2024}
}

@book{bertsekas1996stochastic,
  title={Stochastic optimal control: the discrete-time case},
  author={Bertsekas, Dimitri and Shreve, Steven E},
  volume={5},
  year={1996},
  publisher={Athena Scientific}
}

@inproceedings{ng2013pegasus,
author = {Ng, Andrew Y. and Jordan, Michael}, title = {PEGASUS: a policy search method for large MDPs and POMDPs}, year = {2000}, isbn = {1558607099}, publisher = {Morgan Kaufmann Publishers Inc.}, address = {San Francisco, CA, USA}, booktitle = {Proceedings of the Sixteenth Conference on Uncertainty in Artificial Intelligence}, pages = {406–415}, numpages = {10}, location = {Stanford, California}, series = {UAI'00} 
}

@article{artzner1999coherent,
  title={Coherent measures of risk},
  author={Artzner, Philippe and Delbaen, Freddy and Eber, Jean-Marc and Heath, David},
  journal={Mathematical finance},
  volume={9},
  number={3},
  pages={203--228},
  year={1999},
  publisher={Wiley Online Library}
}

@article{rockafellar2000optimization,
  title={Optimization of conditional value-at-risk},
  author={Rockafellar, R Tyrrell and Uryasev, Stanislav and others},
  journal={Journal of risk},
  volume={2},
  pages={21--42},
  year={2000}
}

@article{jaakkola1993convergence,
  title={Convergence of stochastic iterative dynamic programming algorithms},
  author={Jaakkola, Tommi and Jordan, Michael and Singh, Satinder},
  journal={Advances in neural information processing systems},
  volume={6},
  year={1993}
}

@article{tsitsiklis1994asynchronous,
  title={Asynchronous stochastic approximation and Q-learning},
  author={Tsitsiklis, John N},
  journal={Machine learning},
  volume={16},
  number={3},
  pages={185--202},
  year={1994},
  publisher={Springer}
}

@inproceedings{bertsekas1995neuro,
  title={Neuro-dynamic programming: an overview},
  author={Bertsekas, Dimitri P and Tsitsiklis, John N},
  booktitle={Proceedings of 1995 34th IEEE conference on decision and control},
  volume={1},
  pages={560--564},
  year={1995},
  organization={IEEE}
}

@book{boucheron2013,
    author = {Boucheron, Stéphane and Lugosi, Gábor and Massart, Pascal},
    title = {Concentration Inequalities: A Nonasymptotic Theory of Independence},
    publisher = {Oxford University Press},
    year = {2013},
    month = {02},
    isbn = {9780199535255},
    doi = {10.1093/acprof:oso/9780199535255.001.0001},
    url = {https://doi.org/10.1093/acprof:oso/9780199535255.001.0001},
}

@inproceedings{scherrer2010should,
  title={Should one compute the Temporal Difference fix point or minimize the Bellman Residual? The unified oblique projection view},
  author={Scherrer, Bruno},
  booktitle={27th International Conference on Machine Learning-ICML 2010},
  year={2010}
}

@article{robbins1951stochastic,
  title={A stochastic approximation method},
  author={Robbins, Herbert and Monro, Sutton},
  journal={The annals of mathematical statistics},
  pages={400--407},
  year={1951},
  publisher={JSTOR}
}

\newpage

\onecolumn

\title{Joint MDPs and Distributional Reinforcement Learning in Coupled-Dynamics Environments\\(Supplementary Material)}
\maketitle
\appendix
\section{Proofs}\label{app:proofs}
\begin{lemma}
Let $\lambda := 2/(1-\gamma)$. For any moment collection $M$, define the norm
\begin{equation}
\norm{M}_{\lambda} := \max \left\{\norm{M_\mu}_\infty, \frac{1}{\lambda}\norm{M_\Sigma}_\infty \right\}.
\end{equation}
Then, $T^\pi_2$ is a $\gamma$-contraction in $\norm{\,\cdot\,}_{\lambda}$.
\end{lemma}
\begin{proof}
Firstly, $\norm{\, \cdot \,}_{\lambda}$ is clearly a norm, since it is a maximum over two norms in their respective coordinate spaces. 

Now, we have, for any $(s, a) \in \Ss \times \A$,
\begin{equation}
\begin{split}
\left\lvert (T^\pi_2M)_\mu(s, a) -(T^\pi_2 M')_\mu(s,a)\right\rvert = \gamma\left\lvert \E\big[(M-M')_\mu\big(S'^{(a)}, A'^{(a)}\big) \mid S = s\big]\right\rvert \le\gamma \LVert(M-M')_\mu\RVert_\infty.
\end{split}
\end{equation}
Then, taking the supremum over $(s,a)$ on the left-hand side directly yields
\begin{equation}\label{eq:contraction2first}
\begin{split}
\LVert (T^\pi_2M - T^\pi_2M')_\mu\RVert_\infty \le \gamma \LVert(M-M')_\mu\RVert_\infty &\le \gamma \max \left\{ \LVert (M-M')_\mu\RVert_\infty, \frac{1}{\lambda}\LVert (M-M')_\Sigma\RVert_\infty\right\} \\
&= \gamma \LVert M-M'\RVert_{\lambda}
\end{split}
\end{equation}
Similarly for the second-moment coordinate space, for any $(s, a, \Tilde{s}, \Tilde{a}) \in (\Ss \times \A)^2$,
\begin{equation}
\begin{gathered}
\Lvert(T^\pi_2 M)_\Sigma(s, a, \Tilde{s}, \Tilde{a}) - (T^\pi_2 M')_\Sigma(s, a, \Tilde{s}, \Tilde{a})\Rvert = \big\lvert \E \big[\gamma R^{(a)} (M-M')_\mu \big(\Tilde{S}'^{(\Tilde{a})}, \Tilde{A}'^{(\Tilde{a})}\big) + \gamma \Tilde{R}^{(\Tilde{a})}(M-M')_\mu\big(S'^{(a)}, A'^{(a)}\big) \\ + \gamma^2 (M-M')_\Sigma \big(S'^{(a)}, A'^{(a)}, \Tilde{S}'^{(\Tilde{a})}, \Tilde{A}'^{(\Tilde{a})}\big)\mid (S, \Tilde{S}) = (s, \Tilde{s}) \big]\rvert \\
\le  \E \big[\gamma \Lvert R^{(a)} (M-M')_\mu \big(\Tilde{S}'^{(\Tilde{a})}, \Tilde{A}'^{(\Tilde{a})}\big)\Rvert + \gamma \Lvert\Tilde{R}^{(\Tilde{a})}(M-M')_\mu\big(S'^{(a)}, A'^{(a)}\big)\Rvert \\ + \gamma^2 \Lvert (M-M')_\Sigma \big(S'^{(a)}, A'^{(a)}, \Tilde{S}'^{(\Tilde{a})}, \Tilde{A}'^{(\Tilde{a})}\big)\Rvert\mid (S, \Tilde{S}) = (s, \Tilde{s}) \big]  \\
\le 2\gamma \LVert (M-M')_\mu\RVert_\infty + \gamma^2 \LVert (M-M')_\Sigma \RVert_\infty.
\end{gathered}
\end{equation}
Once again, taking the supremum over $(s,a, \Tilde{s}, \Tilde{a})$ on the left-hand side and dividing by $\lambda$ yields
\begin{equation}\label{eq:contraction2second}
\begin{split}
\frac{1}{\lambda}\LVert (T^\pi_2 M - T^\pi_2 M')_\Sigma)\RVert_\infty &\le \frac{2\gamma}{\lambda} \LVert (M-M')_\mu\RVert_\infty + \frac{\gamma^2}{\lambda} \LVert (M-M')_\Sigma \RVert_\infty \\
&\le \big(\frac{2\gamma}{\lambda} + \gamma^2 \big)\max \left\{ \LVert (M-M')_\mu\RVert_\infty, \frac{1}{\lambda}\LVert (M-M')_\Sigma\RVert_\infty\right\} \\
&=\gamma  \max \left\{ \LVert (M-M')_\mu\RVert_\infty, \frac{1}{\lambda}\LVert (M-M')_\Sigma\RVert_\infty\right\} \\
&= \gamma \LVert M-M'\RVert_{\lambda}.
\end{split}
\end{equation}
Using \eqref{eq:contraction2first} and \eqref{eq:contraction2second}, we conclude
\begin{equation}
\LVert T^\pi_2 M - T^\pi_2 M'\RVert_{\lambda} := \max\left\{\LVert (T^\pi_2M - T^\pi_2M')_\mu\RVert_\infty,  \frac{1}{\lambda}\LVert (T^\pi_2 M - T^\pi_2 M')_\Sigma)\RVert_\infty\right\} \le \gamma \LVert M-M'\RVert_{\lambda}.
\end{equation}
\end{proof}
\begin{theorem}
$T^\pi_2$ admits a unique fixed point $M^\pi_2$. Moreover, for $\lambda = 2/(1-\gamma)$ and initialization $M_0$, JIPE-$2$ has geometric convergence in $\gamma$ with respect to $\norm{\,\cdot\,}_{\lambda}$:
\begin{equation}
\LVert M_k - M^\pi_2\RVert_{\lambda} \le \gamma^k\LVert M_0 - M^{\pi}_2 \RVert_{\lambda}.
\end{equation}
Furthermore, for any moment collection $M$, define the $2$nd-order Bellman residual as $\LVert M - T^\pi_2 M\RVert_{\lambda}$.
Then,
\begin{equation}
\LVert M - M^\pi_2 \RVert_{\lambda} \le \frac{1}{1-\gamma}\LVert M - T^\pi_2 M\RVert_{\lambda}.
\end{equation}
\end{theorem}
\begin{proof}
The proof of the first statement is a simple application of the Banach fixed-point theorem to Lemma \ref{lem:2contraction}. For the second statement, since $M^\pi_2 = T^\pi_2 M^\pi_2$,
\begin{equation}
\begin{split}
\LVert M - M^\pi_2 \RVert_{\lambda} &\le \LVert M - T^\pi_2 M\RVert_{\lambda} + \LVert T^\pi_2 M - T^\pi_2 M^\pi_2\RVert_{\lambda} \\
&\le \LVert M - T^\pi_2 M\RVert_{\lambda} + \gamma \LVert M - M^\pi_2\RVert_{\lambda}.
\end{split}
\end{equation}
Rearranging furnishes the result.
\end{proof}
\begin{lemma}\label{lem:noises}
Let $\mathcal{F}_k$ be a filtration such that each update index $I_k \in \mathcal{I}_2$ is $\mathcal{F}_k$-measurable, and the randomness used to form $\hat{T}^\pi_2(M_k, I_k)$ is $\mathcal{F}_{k+1}$-measurable and conditionally independent of the past given $\mathcal{F}_k$. Let the coordinatewise noise term be 
\begin{equation}
\omega_{k+1}(i) := \hat{T}^\pi_2(M_k, i) - (T^\pi_2M_k)(i).
\end{equation}
Then, for every index $i \in \mathcal{I}_2$ updated at time $k$
\begin{equation}\label{eq:noisemean}
\begin{gathered}
\E\big[\omega_{k+1}(i)\mid \mathcal{F}_k \big] = 0,\\
\end{gathered}
\end{equation}
and there exist finite constants $C_0, C_1>0$ such that
\begin{equation}\label{eq:noise2}
\E\big[\omega_{k+1}(i)^2\mid \mathcal{F}_k \big] \le C_0 + C_1 \LVert M_k \RVert_\lambda^2.
\end{equation}
\end{lemma}
\begin{proof}
Fix a time $k$ and let $i \in \mathcal{I}_2$ be an index updated at time $k$. Equation \eqref{eq:noisemean} holds by construction, since $\hat{T}^\pi_2(M_k, i)$ is a sample-based estimator of the defining conditional expectations for $(T^\pi_2M
_k)(i)$. In other words, $(T^\pi_2 M_k)(i) = \E[\hat{T}^\pi_2(M_k,i)\mid \mathcal{F}_k]$. It remains to prove \eqref{eq:noise2}. Firstly,
\begin{equation}
\begin{split}
\omega_{k+1}(i)^2 &= \big(\hat{T}^\pi_2(M_k, i) - (T^\pi_2M_k)(i)\big)^2 \\
&\le 2\hat{T}^\pi_2(M_k, i)^2 +2(T^\pi_2M_k)(i)^2.
\end{split}
\end{equation}
Taking expectations conditional to $\mathcal{F}_k$ on both sides,
\begin{equation}\label{eq:4times}
\begin{split}
\E\big[ \omega_{k+1}(i)^2 \mid \mathcal{F}_k\big] &\le 2\E\big[\hat{T}^\pi_2(M_k, i)^2 \mid \mathcal{F}_k\big] + 2(T^\pi_2M_k)(i)^2 \\
&\le 2\E\big[\hat{T}^\pi_2(M_k, i)^2 \mid \mathcal{F}_k\big] + 2\E\big[\hat{T}^\pi_2(M_k, i) \mid \mathcal{F}_k \big] \\
&= 4\E\big[\hat{T}^\pi_2(M_k, i)^2 \mid \mathcal{F}_k\big],
\end{split}
\end{equation}
using Jensen's inequality. Thus, it suffices to upper bound $\E[\hat{T}^\pi_2(M_k, i) \mid \mathcal{F}_k]$ by an affine function of $\norm{M_k}^2_\lambda$ for the two cases of $i$ being a $\mu$-index or a $\Sigma$-index.
\begin{enumerate}[leftmargin=*, itemsep=2pt, align=left]
\item[\textbf{Case 1:}] If $i = ((s, a), \mu)$, then by definition of the sample backup, 
\begin{equation}
\hat{T}^\pi_2(M_k, i) = R^{(a)} + \gamma (M_k)_\mu\big(S'^{(a)}, A'^{(a)}\big),
\end{equation}
with $R^{(a)} \in [0, 1]$ almost surely. Hence,
\begin{equation}
\Lvert \hat{T}^\pi_2(M_k, i)\Rvert \le 1 + \gamma \LVert (M_k)_\mu \RVert_\infty \le 1 + \gamma \LVert M_k \RVert_\lambda,
\end{equation}
or alternatively,
\begin{equation}
\hat{T}^\pi_2(M_k, i)^2 \le 2 + 2 \gamma^2 \LVert M_k\RVert^2_\lambda.
\end{equation}
Taking conditional expectations with respect to $\mathcal{F}_k$, and using \eqref{eq:4times}, we have
\begin{equation}\label{eq:c1-1}
\E\big[\omega_{k+1}(i)^2 \mid \mathcal{F}_k\big] \le 8 + 8\gamma^2\LVert M_k\RVert^2_\lambda.
\end{equation}
\item[\textbf{Case 2:}] If $i = ((s, a), (\Tilde{s}, \Tilde{a}), \Sigma)$, we have the sample backup
\begin{equation}
\hat{T}^\pi_2(M, i):=R^{(a)}\Tilde{R}^{(\Tilde{a})}+\gamma R^{(a)} M_\mu\big(\Tilde{S}'^{(\Tilde{a})},\Tilde{A}'^{(\Tilde{a})}\big)+  \gamma \Tilde{R}^{(\Tilde{a})} M_\mu\big(S'^{(a)},A'^{(a)}\big) + \gamma^2 M_\Sigma\big(S'^{(a)}, A'^{(a)}, \Tilde{S}'^{(\Tilde{a})}, \Tilde{A}'^{(\Tilde{a})}\big).
\end{equation}
\end{enumerate}
Once again, $R^{(a)}, \Tilde{R}^{(\Tilde{a})} \in [0, 1]$ almost surely. Then,
\begin{equation}
\begin{split}
\Lvert \hat{T}^\pi_2(M, i) \Rvert &\le 1 + \gamma \Lvert (M_k)_\mu\big(\Tilde{S}'^{(\Tilde{a})},\Tilde{A}'^{(\Tilde{a})}\big)\Rvert + \gamma \Lvert(M_k)_\mu\big(S'^{(a)},A'^{(a)}\big) \Rvert + \gamma^2 \Lvert(M_k)_\Sigma\big(S'^{(a)}, A'^{(a)}, \Tilde{S}'^{(\Tilde{a})}, \Tilde{A}'^{(\Tilde{a})} \big) \Rvert \\
&\le 1 + 2\gamma \LVert (M_k)_\mu \RVert_\infty + \gamma^2 \LVert (M_k)_\Sigma \RVert_\infty.
\end{split}
\end{equation}
Furthermore, we have $\norm{(M_k)_\mu}_\infty \le \norm{M_k}_\lambda$ and $\norm{(M_k)_\Sigma}_\infty \le \lambda\norm{M_k}_\lambda$. Then,
\begin{equation}
\Lvert \hat{T}^\pi_2(M_k, i) \Rvert \le 1+ (2\gamma + \gamma^2\lambda) \LVert M_k \RVert_\lambda,
\end{equation}
or alternatively,
\begin{equation}
\hat{T}^\pi_2(M_k, i)^2 \le 2 + 2(2\gamma + \gamma^2\lambda)^2\LVert M_k \RVert^2_\lambda.
\end{equation}
Taking conditional expectations with respect to $\mathcal{F}_k$, and using \eqref{eq:4times}, we have
\begin{equation}\label{eq:c1-2}
\E\big[\omega_{k+1}(i)^2 \mid \mathcal{F}_k\big] \le 8 + 8(2\gamma +\gamma^2\lambda)^2\LVert M_k \RVert^2_\lambda.
\end{equation}
Combining \eqref{eq:c1-1} and \eqref{eq:c1-2} to define
\begin{equation}
C_0 := 8, \qquad C_1:= 8 \max\left\{\gamma^2, (2\gamma +\gamma^2\lambda)^2 \right\}
\end{equation}
finishes the proof.
\end{proof}

\begin{theorem}
Assume each coordinate $i \in \mathcal{I}_2$ is selected infinitely often by $I_k$. Let the step size $\alpha_k(i)$ satisfy the conditions of \citet{robbins1951stochastic} for every index $i$:
\begin{equation}
\sum_k \alpha_k(i) = \infty,\quad \sum_k\alpha_k(i)^2 < \infty.
\end{equation} 
Then, the incremental JIPE-$2$ iteration converges almost surely in $\norm{\,\cdot\,}_\lambda$, i.e.,
\begin{equation}
\LVert M_k - M^\pi_2 \RVert_\lambda \to 0 \;\;\text{almost surely as} \; k \to \infty.
\end{equation}
\end{theorem}
\begin{proof}
By Lemma \ref{lem:2contraction}, $T^\pi_2$ is a contraction map on the space of moment collections under $\norm{\,\cdot\,}_\lambda$, with unique fixed point $M^\pi_2$ (Theorem \ref{thm:jipe2}). The incremental JIPE-$2$ recursion is an asynchronous stochastic approximation scheme \citep{jaakkola1993convergence, tsitsiklis1994asynchronous, bertsekas1995neuro} for $T^\pi_2$, with martingale-difference noise and controlled conditional second moment (Lemma \ref{lem:noises}). The step size and visitation conditions are those required by the contraction-map stochastic approximation theorem by \citet{DRL-textbook}, yielding almost-sure convergence to the fixed point.
\end{proof}

\newpage
\section{JIPE-2 with Function Approximation}\label{app:funcapprox}
The exact DP algorithm, namely JIPE-$2$, is restricted to environments with discrete and appropriately small state-action spaces $\mathcal{X}$. In high-dimensional or continuous state spaces, representing the true moment collection $M=(M_\mu,M_\Sigma)$ exactly may become computationally intractable. We therefore study a projected fixed-point formulation in which the candidate moment collection is restricted to a parameterized family
\begin{equation}
\hat{M}(\theta) = \big(\hat{M}_\mu(\theta_\mu),\hat{M}_\Sigma(\theta_\Sigma)\big).
\end{equation}

\paragraph{Function classes.}
For the first moment, let $\phi_\mu:\mathcal{X}\to\mathbb{R}^{d_\mu}$ be a feature map and define
\begin{equation}
\hat{M}_\mu(x;\theta_\mu):=\phi_\mu(x)^\top\theta_\mu,\qquad x=(s,a)\in\mathcal{X}.
\end{equation}
For the second moment, we use a PSD-preserving quadratic form
\begin{equation}
\hat{M}_\Sigma(x,y;\Theta_\Sigma) := \phi_\mu(x)^\top\Theta_\Sigma\phi_\mu(y),\qquad \Theta_\Sigma\succeq 0.
\end{equation}
Let $\mathcal{K}_\Sigma$ denote the closed convex cone of second-moment functions representable in this form. This parameterization ensures that every finite Gram matrix produced by $\hat{M}_\Sigma$ is positive semidefinite, so the approximation respects the second-moment geometry of return random variables.

\begin{assumption}\label{ass:rank}
The feature matrix $\Phi_\mu\in\mathbb{R}^{|\mathcal{X}|\times d_\mu}$ has full column rank.
\end{assumption}

\begin{assumption}\label{ass:ergodic}
The Markov chain over $\mathcal{X}$ induced by the stationary policy $\pi$ is ergodic, with unique stationary state-action distribution $\nu$ satisfying $\nu(x)>0$ for all $x\in\mathcal{X}$.
\end{assumption}

For first moments, we use the standard on-policy geometry
\begin{equation}\label{eq:inner1}
\langle f,g\rangle_\nu := \sum_{x\in\mathcal{X}}\nu(x)f(x)g(x), \qquad \|f\|_\nu := \sqrt{\langle f,f\rangle_\nu}.
\end{equation}

For second moments, the natural transition operator is the pair kernel $P_2$ on $\mathcal{X}^2$ induced by the JMDP coupling and the fixed policy $\pi$. Since $\mathcal{X}^2$ is finite, $P_2$ has at least one stationary distribution. Fix any such distribution and call it $\mu_2$, so that
\begin{equation}
\mu_2 P_2=\mu_2.
\end{equation}
The marginal dynamics of each coordinate of the pair process agree with the policy-induced marginal kernel $P^\pi$. By Assumption \ref{ass:ergodic}, the marginals of any stationary $\mu_2$ must therefore equal $\nu$. The distribution $\mu_2$ need not be unique and need not have full support. Let
\begin{equation}
\operatorname{supp}(\mu_2) := \{(x,y)\in\mathcal{X}^2:\mu_2(x,y)>0\}.
\end{equation}
All $L_2(\mu_2)$ statements below are therefore support-relative, equivalently, stated modulo functions that agree on $\operatorname{supp}(\mu_2)$. Define
\begin{equation}\label{eq:inner2}
\langle h,\ell\rangle_{\mu_2} := \sum_{x,y\in\mathcal{X}}\mu_2(x,y)h(x,y)\ell(x,y),\qquad \|h\|_{\mu_2}:=\sqrt{\langle h,h\rangle_{\mu_2}}.
\end{equation}

\paragraph{Projection.}
We use a block-diagonal projection $\Pi M = (\Pi_\mu M_\mu,\Pi_\Sigma M_\Sigma)$, where
\begin{equation}\label{eq:least-squares}
\Pi_\mu M_\mu := \arg\min_{\hat{M}_\mu\in\mathrm{span}(\Phi_\mu)}\|M_\mu-\hat{M}_\mu\|_\nu
\end{equation}
and
\begin{equation}\label{eq:psd-projection}
\Pi_\Sigma M_\Sigma := \arg\min_{\hat{M}_\Sigma\in\mathcal{K}_\Sigma}\|M_\Sigma-\hat{M}_\Sigma\|_{\mu_2}.
\end{equation}
The first projection is the usual orthogonal projection. The second is projection onto a closed convex cone in the Hilbert space $L_2(\mu_2)$, hence is non-expansive. If $\mu_2$ is not full support, the projection does not identify values outside $\operatorname{supp}(\mu_2)$. The projected fixed-point equation is
\begin{equation}\label{eq:projection}
  \hat{M}^\ast=\Pi T^\pi_2\hat{M}^\ast.
\end{equation}

\subsection{Projected JIPE-2}
Projected JIPE-$2$ iterates
\begin{equation}
M_{k+1}=\Pi T^\pi_2 M_k.
\end{equation}
In parameters, the first-moment update has the usual least-squares form
\begin{equation}
\theta_{\mu,k+1} = \big(\Phi_\mu^\top D_\nu\Phi_\mu\big)^{-1} \Phi_\mu^\top D_\nu T^\pi_{2,\mu}\hat{M}(\theta_k),
\end{equation}
where $D_\nu$ is diagonal with entries $\nu(x)$. The second-moment update is the PSD-constrained projection
\begin{equation}
\Theta_{\Sigma,k+1} = \arg\min_{\Theta\succeq 0} \left\|\Phi_\mu\Theta\Phi_\mu^\top - T^\pi_{2,\Sigma}\hat{M}(\theta_k)\right\|_{D_{\mu_2}}^2,
\end{equation}
where $D_{\mu_2}$ is diagonal over $\mathcal{X}^2$ with entries $\mu_2(x,y)$, and $T^\pi_{2,\Sigma}\hat{M}(\theta_k)$ is viewed as a $|\mathcal{X}|\times|\mathcal{X}|$ matrix indexed by pair coordinates.

\subsection{Analysis of Projected JIPE-2}
In the pair-stationary geometry, the coupled pair transition is automatically non-expansive in the norm used for the second-moment projection. The next lemma formalizes this.

\begin{lemma}\label{lem:non-expansive}
The marginal transition operator $P^\pi$ is non-expansive in $L_2(\nu)$, and the pair transition operator $P_2$ is non-expansive in $L_2(\mu_2)$.
\end{lemma}
\begin{proof}
For $P^\pi$, by Jensen's inequality and stationarity of $\nu$,
\begin{equation}
\begin{split}
\|P^\pi f\|_\nu^2 &= \sum_x\nu(x)\left(\sum_{x'}P^\pi(x'\mid x)f(x')\right)^2\\
&\le \sum_x\nu(x)\sum_{x'}P^\pi(x'\mid x)f(x')^2 = \sum_{x'}\nu(x')f(x')^2 = \|f\|_\nu^2.
\end{split}
\end{equation}
The proof for $P_2$ is identical, using stationarity of $\mu_2$:
\begin{equation}
\begin{split}
\|P_2 h\|_{\mu_2}^2 &= \sum_z\mu_2(z)\left(\sum_{z'}P_2(z'\mid z)h(z')\right)^2 \\
&\le \sum_z\mu_2(z)\sum_{z'}P_2(z'\mid z)h(z')^2 = \sum_{z'}\mu_2(z')h(z')^2 =\|h\|_{\mu_2}^2,
\end{split}
\end{equation}
where $z=(x,y)\in\mathcal{X}^2$.
\end{proof}

\begin{theorem}
Under Assumptions \ref{ass:rank} and \ref{ass:ergodic}, for any stationary distribution $\mu_2$ of $P_2$, there exists $\beta > 0$ such that the projected joint Bellman operator $\Pi T^\pi_2$ is a $\kappa$-contraction in the weighted norm
\begin{equation}
\|M\|_{\nu,\mu_2,\beta} := \max\left\{\|M_\mu\|_\nu,\beta\|M_\Sigma\|_{\mu_2}\right\},
\end{equation}
with some $\kappa < 1$. Consequently, the projected iteration $M_{k+1} = \Pi T^\pi_2M_k$ converges to a unique projected fixed point $\hat{M}^\ast$ in this norm, equivalently a unique fixed point on $\operatorname{supp}(\mu_2)$ for the second-moment coordinate, and
\begin{equation}
\|\hat{M}^\ast - M^\pi_2\|_{\nu,\mu_2,\beta} \le \frac{1}{1-\kappa} \|\Pi M^\pi_2 - M^\pi_2\|_{\nu,\mu_2,\beta}.
\end{equation}
Moreover, for any pair $(x,y)$ with $\mu_2(x,y) > 0$,
\begin{equation}
\left|\hat{M}^\ast_\Sigma(x,y) - M^\pi_{2,\Sigma}(x,y)\right| \le \frac{1}{\beta\sqrt{\mu_2(x,y)}}\, \|\hat{M}^\ast - M^\pi_2\|_{\nu,\mu_2,\beta}.
\end{equation}
Thus, the theorem provides a pointwise guarantee for every pair state in the support of the chosen stationary pair distribution. If $\mu_2$ has full support, this becomes a guarantee over all pair states.
\end{theorem}

\begin{proof}
Let $M$ and $\bar{M}$ be arbitrary moment collections and write $\Delta = M - \bar{M}$. By non-expansiveness of $\Pi_\mu$ and Lemma \ref{lem:non-expansive},
\begin{equation}
\|\Pi_\mu(T^\pi_{2,\mu}M-T^\pi_{2,\mu}\bar{M})\|_\nu \le \gamma\|\Delta_\mu\|_\nu.
\end{equation}

For the second moment, for $X = (s,a)$ and $Y = (\tilde{s},\tilde{a})$,
\begin{equation}
(T^\pi_{2,\Sigma}\Delta)(X,Y) = \gamma\E[R^{(a)}\Delta_\mu(Y')\mid X,Y] + \gamma\E[\tilde{R}^{(\tilde{a})}\Delta_\mu(X')\mid X,Y] + \gamma^2\E[\Delta_\Sigma(X',Y')\mid X,Y].
\end{equation}
The rewards are in $[0,1]$, so Jensen's inequality and the fact that the $Y'$ marginal under stationary pair dynamics is $\nu$ imply
\begin{equation}
\begin{split}
\left\|\E[| \Delta_\mu(Y') |\mid X,Y]\right\|_{\mu_2}^2 &\le \E_{\mu_2}\left[\E[\Delta_\mu(Y')^2\mid X,Y]\right] \\
&= \E_{\mu_2}[\Delta_\mu(Y')^2] = \|\Delta_\mu\|_\nu^2.
\end{split}
\end{equation}
The same bound holds for the symmetric term involving $X'$. For the final term, we use Lemma \ref{lem:non-expansive} to get
\begin{equation}
\|P_2\Delta_\Sigma\|_{\mu_2}\le \|\Delta_\Sigma\|_{\mu_2}.
\end{equation}
Since $\Pi_\Sigma$ is non-expansive in $L_2(\mu_2)$,
\begin{equation}
\|\Pi_\Sigma(T^\pi_{2,\Sigma}M - T^\pi_{2,\Sigma}\bar{M})\|_{\mu_2} \le 2\gamma\|\Delta_\mu\|_\nu+\gamma^2\|\Delta_\Sigma\|_{\mu_2}.
\end{equation}
Therefore
\begin{equation}
\|\Pi T^\pi_2M-\Pi T^\pi_2\bar{M}\|_{\nu,\mu_2,\beta} \le \max\left\{\gamma, 2\beta\gamma + \gamma^2 \right\}\|\Delta\|_{\nu,\mu_2,\beta}.
\end{equation}
Choose any
\begin{equation}
0 < \beta < \frac{1-\gamma^2}{2\gamma}.
\end{equation}
Then $\kappa := \max\{\gamma,2\beta\gamma + \gamma^2\} < 1$. Banach's fixed-point theorem guarantees convergence to a unique fixed point in the weighted norm. The approximation bound follows from the standard projected Bellman argument:
\begin{equation}
\begin{split}
\|\hat{M}^\ast - M^\pi_2\| &= \|\Pi T^\pi_2\hat{M}^\ast - T^\pi_2M^\pi_2\| \\
&\le \|\Pi T^\pi_2\hat{M}^\ast - \Pi T^\pi_2M^\pi_2\| + \|\Pi M^\pi_2-M^\pi_2\| \\
&\le \kappa\|\hat{M}^\ast - M^\pi_2\| + \|\Pi M^\pi_2 - M^\pi_2\|,
\end{split}
\end{equation}
where all norms are $\|\cdot\|_{\nu,\mu_2,\beta}$. Rearranging yields the claim.
\end{proof}

\newpage
\section{ALE implementation and gap diagnostics}\label{app:ale}

The ALE experiments use the same multi-action query interface as the tabular experiments, implemented by cloning the emulator state, restoring it for each queried action, applying shared exogenous shocks across the queried actions, and advancing only the executed branch. In the reward-coupled diagnostic below, the shared exogenous shock is the reward perturbation used to construct a nontrivial same-scenario gap law while preserving a controlled evaluation setting.

For neural second-moment approximation, we parameterize the covariance component through a Gram representation,
\begin{equation}
M_\Sigma(x,y) = M_\mu(x)M_\mu(y) + \langle v(x),v(y)\rangle,
\end{equation}
which ensures that the estimated covariance matrix is positive semidefinite. Training also minimizes a TD residual for the implied gap second moment,
\begin{equation}
E[(Z_a - Z_b)^2] = M_\Sigma(x_a,x_a) + M_\Sigma(x_b,x_b) - 2M_\Sigma(x_a,x_b).
\end{equation}
thereby directly supervising the combination of second-moment entries that determines action-gap variance.

\begin{table}[h]
\centering
\caption{\textbf{Reward-coupled ALE gap-variance validation.}
For each game, we sample emulator states, estimate the action-gap variance from recursive-pair MC, and report the ratios of the learned coupled prediction and an independent-marginals prediction to the same MC estimate. Values near $1.0$ represent better calibrated estimates to the coupled MC gap variance.}
\label{tab:ale-gap-ratios}
\begin{tabular}{lcc}
\toprule
Game & Coupled / MC & Independent / MC \\
\midrule
Atlantis & 1.03 & 83.3 \\
BattleZone & 3.87 & 115.7 \\
Boxing & 3.97 & 111.1 \\
Pong & 1.30 & 135.4 \\
\bottomrule
\end{tabular}
\end{table}

Across the four games, the coupled estimates are within a factor of $1.03$--$3.97$ of the recursive-pair MC gap variance, while the independent-marginals estimates severely overestimate it by factors of $83.3$--$135.4$.
\end{document}